\documentclass{article}

\PassOptionsToPackage{numbers, compress}{natbib}

\usepackage[preprint]{neurips_2021}

\usepackage[utf8]{inputenc} 
\usepackage[T1]{fontenc}    
\usepackage{hyperref}       
\usepackage{url}            
\usepackage{booktabs}       
\usepackage{amsfonts}       
\usepackage{nicefrac}       
\usepackage{microtype}      
\usepackage{xcolor}         
\usepackage{multirow}

\usepackage{hyperref}  
\hypersetup{
    colorlinks=true,
    linkcolor=blue,
    citecolor =blue,
    filecolor=magenta,      
    urlcolor=magenta,
}
\usepackage{url}            
\usepackage{booktabs}       
\usepackage{amsfonts}       
\usepackage{nicefrac}       
\usepackage{microtype}      
\usepackage{lipsum}
\usepackage{graphicx}

\usepackage{color}

\usepackage{makecell}

\usepackage{amsmath}
\usepackage{amsfonts}
\usepackage[ruled,vlined]{algorithm2e}
\usepackage{bm}
\usepackage{amssymb}
\usepackage{color}
\usepackage{xcolor}
\usepackage{hyperref}
\usepackage{amsmath, amssymb, graphicx, url, amsthm}
\usepackage{cleveref}

\def\u{\tilde{u}}

\def\X{\mathcal{X}}

\def\ub{\mathbf{u}}
\def\Zb{\mathbf{Z}}
\def\wb{\mathbf{w}}

\def\mb{\mathbf{m}}

\def\H{\mathcal{H}}
\def\Sb{\mathbf{S}}
\def\Sc{\mathcal{S}}

\def\Nc{\mathcal{N}}
\def\Ec{\mathcal{E}}

\def\Bc{\mathcal{B}}
\def\Ib{\mathbf{I}}
\def\phib{\bm{\phi}}
\def\Ic{\mathcal{I}}

\def\argmax{\footnotesize \mbox{argmax}}
\def\argmin{\footnotesize \mbox{argmin}}
\def\TP{\footnotesize \mbox{T}}

\def\phib{\bm{\phi}}

\def\Phib{\bm{\Phi}}

\def\Xb{\mathbf{X}}
\def\yb{\mathbf{y}}
\def\Ib{\mathbf{I}}

\def\Rr{\mathbb{R}}
\def\Nn{\mathbb{N}}
\def\E{\mathbb{E}}

\newtheorem{lemma}{Lemma}
\newtheorem{theorem}{Theorem}
\newtheorem{proposition}{Proposition}

\newtheorem{assumption}{Assumption}

\def\nn{\nonumber}

\title{Scalable Thompson Sampling using\\Sparse Gaussian Process Models}

\author{%
  Sattar Vakili\thanks{Equal contribution,
$^1$ MediaTek Research, $^2$ Secondmind, $^{3}$ Imperial College London, $^{4}$ University of Cambridge. Correspondence to Sattar Vakili $<$\texttt{sattar.vakili@mtkresearch.com}$>$, Henry Moss $<$\texttt{henry.moss@secondmind.ai}$>$.

}$^{~~1}$,~~Henry Moss$^{*2}$,~~Artem Artemev$^{2,3}$,~~Vincent Dutordoir$^{2,4}$,~~Victor Picheny$^{2}$ \\

}

\begin{document}

\maketitle

\begin{abstract}

Thompson Sampling (TS) from Gaussian Process (GP) models is a powerful tool for the optimization of black-box functions. Although TS enjoys strong theoretical guarantees and convincing empirical performance, it incurs a large computational overhead that scales polynomially with the optimization budget. Recently, scalable TS methods based on sparse GP models have been proposed to increase the scope of TS, enabling its application to problems that are sufficiently multi-modal, noisy or combinatorial to require more than a few hundred evaluations to be solved. 
However, the approximation error introduced by sparse GPs invalidates all existing regret bounds. In this work, we perform a theoretical and empirical analysis of scalable TS. We provide theoretical guarantees and show that the drastic reduction in computational complexity of scalable TS can be enjoyed without loss in the regret performance over the standard TS. These conceptual claims are validated for practical implementations of scalable TS on synthetic benchmarks and as part of a real-world high-throughput molecular design task.

\end{abstract}

\section{Introduction}\label{Intro}

Thompson sampling \citep[TS,][]{Thompson1933} is a popular algorithm for Bayesian optimization \citep[BO,][]{Shahriari2016} --- a sequential model-based approach for the optimization of expensive-to-evaluate black-box functions, typically characterised by limited prior knowledge and access to only a limited number of (possibly noisy) evaluations. By sequentially evaluating the maxima of random samples from a model of the objective function, TS provides a conceptually simple method for balancing exploration and exploitation. 

TS is often paired with Gaussian Processes (GPs), which offers a spectrum of powerful and flexible modeling tools that provide probabilistic predictions of the objective function. The resulting GP-TS algorithms \citep{Chowdhury2017bandit} have been found to provide highly efficient optimization under heavily restricted optimization budgets, with numerous successful applications  including aerodynamic design \citep{baptista2018bayesian}, route planning \citep{eriksson2021scalable} and web-streaming \citep{daulton2019thompson}. While most popular BO algorithms cannot query more than a handful of points at a time \citep{chevalier2013fast,gonzalez2016batch,wu2016parallel,moss2021gibbon} without employing replicating designs \citep[see ][]{jalali2017comparison,binois2019replication}, TS has a natural ability to query large batches of points. Therefore, TS is a popular solution for optimization pipelines enjoying a large degree of parallelisation, for example in high-throughout chemical space exploration \citep{hernandez2017parallel} and for the distributed tuning of machine learning models across cloud compute resources  \citep{kandasamy2018parallelised}. 

As BO incurs a substantial computational overhead between successive iterations, while updating models and choosing the next set of query points, standard BO methods are limited to optimization problems with small evaluation budgets \citep{Shahriari2016}. However, with large batches, the computational overhead incurred by BO per individual function evaluation is considerably reduced. Therefore, considering large batches is a promising tactic to expand BO to larger optimization budgets, which are required to optimize highly noisy problems with rougher optimization landscapes \citep{jalali2017comparison,binois2019replication} or high dimensional and combinatorial search spaces \citep{snoek2015scalable,hernandez2017parallel,wang2018batched}. Consequently, the highly-parallelizable TS is a promising candidate for BO under large optimization budgets.

Unfortunately, practical implementations of GP-TS suffer from two key computational bottlenecks that prevent the method from scaling in terms of total optimization budget. Not only does each update of the GP posterior distribution require a matrix inversion that incurs a cubic cost w.r.t. the number of observations $t$ \citep{Rasmussen2006}, but even sampling from this posterior can be a daunting task --- the standard approach of drawing a joint sample across a $N$ point discretization of the search space has an 
$O(N^3)$ complexity \citep[due to a Cholesky decomposition step,][]{diggle1998model}. Alternative existing approaches for BO under large optimization budgets include using Neural Networks in lieu of GPs \citep{snoek2015scalable,hernandez2017parallel} or to use local models \citep{eriksson2019scalable} and ensembles \citep{wang2018batched}.

A natural answer to the scalability issues of GP-TS is to rely on the recent advances in Sparse Variational GP models \citep[SVGP,][]{Titsias2009Variational}. SVGPs provide a low rank $O(m^2 t)$ approximation of the GP posterior, where $m$ is the number of the so-called \emph{inducing variables}
that grows at a rate much slower than $t$.
Successful applications of SVGPs for BO under large optimization budgets include optimizing a free-electron laser  \citep{mcintire2016sparse}, molecules under synthesis-ability constraints \citep{griffiths2017constrained}, and the composition of alloys \citep{yang2019sparse}. 
Furthermore, \cite{wilson2020efficientsampling} introduced an efficient sampling rule (referred to as \emph{decoupled } sampling) which can be used to efficiently perform TS with SVGPs. In particular, \cite{wilson2020efficientsampling} decomposes samples from the SVGP posterior into the sum of an approximate prior based on $M$ features (see Sec.~\ref{Subsec:feature}) and an SVGP model update, thus reducing the computational cost of drawing a Thompson sample to $O\left((m+M)N\right)$. Leveraging this sampling rule results in a scalable GP-TS algorithm (henceforth S-GP-TS) that can handle orders of magnitude greater optimization budgets.

While \cite{Chowdhury2017bandit} proposed a comprehensive theoretical analysis of exact GP-TS, it does not apply to S-GP-TS. Indeed, using sparse models and decoupled sampling introduce two layers of approximation,
that must be handled with care, as even a small constant error in the posterior can lead to poor performance by encouraging under-exploration in the vicinity of the optimum point \cite{Phan2019TSExample}. 
Our primary contributions can be summarised as follows. First, we provide a theoretical analysis showing that batch TS from any approximate GP can achieve the same regret order as an exact GP-TS algorithm as long the quality of the posterior approximations satisfies certain conditions (Assumptions~\ref{Ass1} and~\ref{Ass2}). Second, 
for the specific case of S-GP-TS (batch decoupled TS using a SVGP), we leverage the results of \cite{Burt2019Rates} to provide bounds in terms of GP's kernel spectrum for the number of prior features and  inducing variables required to guarantee low regret.
Finally, we investigate empirically the performance of multiple practical implementations of S-GP-TS, considering synthetic benchmarks and a high-throughput molecular design task. 


\section{Problem Formulation}\label{Sec:PF}

We consider the sequential optimization of an unknown function $f$ over a compact set $\X\subset \Rr^d$. A sequential learning policy selects a batch of $B$ observation points $\{x_{t,b}\}_{b\in [B]}$ at each time step $t=1,2,\dots,T$ and receives the corresponding real-valued and noisy rewards $\{y_{t,b}=f(x_{t,b})+\epsilon_{t,b}\}_{b\in[B]}$, where $\epsilon_{t,b}$ denotes the observation noise. Throughout the paper, we use the notation $[n]=\{1,2,\dots,n\},$ for $n\in \Nn$. 
As is common in both the bandits and GP literature, our analysis uses the following sub-Gaussianity assumption,  a direct consequence of which is that $\E[\epsilon_{t,b}] = 0$, for all~$t,b\in \Nn$.
\begin{assumption}\label{AssNoise}
$\epsilon_{t,b}$ are i.i.d., over both $t$ and $b$, $R-$sub-Gaussian random variables, where $R>0$ is a fixed constant. Specifically,
$
\E[e^{h\epsilon_{t,b}}]\le \exp(\frac{h^2R^2}{2}),~\forall h\in \Rr, \forall t,b\in \Nn. 
$

\end{assumption}

Let $x^*\in\argmax_{x\in\X}f(x)$ be an optimal point.
We can then measure the performance of a sequential optimizer by its \emph{strict regret}, defined as the cumulative loss compared to $f(x^*)$ over a time horizon $T$
\begin{eqnarray} \label{eq:regdef}
R(T,B;f) = \E\left[\sum_{t=1}^T \sum_{b=1}^B f(x^*) - f(x_{t,b})\right],
\end{eqnarray}
where the expectation is with respect to the randomness in noise and the possible stochasticity in the sequence of the selected batch observation points $\{x_{t,b}\}_{t\in[T], b\in [B]}$. 
Note that our regret measure (\ref{eq:regdef}) is defined for the true unknown  $f$. In contrast, the alternative {Bayesian regret} \citep[see e.g.][]{Russo2018TutorialTS, kandasamy2018parallelised} averages over a prior distribution for $f$. As upper bounds on strict regret directly apply to the Bayesian regret (but not necessarily the reverse), our results are stronger than those that can be achieved when analysing just Bayesian regret, for example when applying the technique of \cite{Russo2014} that equates TS's Bayesian regret with that of the well-studied upper confidence bound policies.

Following \cite{Chowdhury2017bandit,srinivas2010gaussian, Calandriello2019Adaptive}, our analysis assumes a regularity condition on the objective function motivated by kernelized learning models and their associated reproducing kernel Hilbert spaces \citep[RKHS,][]{berlinet2011reproducing}: 
\begin{assumption}\label{AssNorm}
Given an RKHS $H_k$, the norm of the objective function is bounded: $||f||_{H_k}\le \Bc$, for some $\Bc>0$, and $k(x,x')\le 1$, for all $x,x'\in\X$.
\end{assumption}

In the case of practically relevant kernels, Assumption~\ref{AssNorm} implies certain smoothness properties for the objective functions.

\section{Gaussian Processes and Sparse Models}\label{Sec:GPs}

GPs are powerful non-parametric Bayesian models over the space of functions \citep{Rasmussen2006} with a distribution specified by a mean function $\mu(x)$
(henceforth assumed to be zero for simplicity)
and a positive definite kernel (or covariance function) $k(x,x')$.
We provide here a brief description of the classical GP model and two sparse variational formulations.

\subsection{Exact Gaussian Process models}

Suppose that we have collected a set of location-observation tuples $\H_{t} = \{\Xb_{t},\yb_{t}\}$, where $\Xb_{t}$ is the $tB\times d$ matrix of locations with rows $[\Xb_{t}]_{(s-1)B+b} = x_{s,b}$, and $\yb_{t}$ is the $tB$-dimensional column vector of observations with elements $[\yb_{t}]_{(s-1)B+b}=y_{s,b}$, for all $s\in[t], b\in[B]$.
Then, assuming a  Gaussian observation noise
, the posterior of the GP model $\hat{f}$ 
given the set of past observations $\H_{t}$, is also a GP with mean $\mu_{t}(\cdot)$, variance $\sigma^2_{t}()$ and kernel function $k_t(\cdot,\cdot)$
specified as
\begin{eqnarray}
\mu_t(x) =  k^{\TP}_{\Xb_t,x} (K_{\Xb_t,\Xb_t}+\tau \Ib)^{-1} \yb_{t},
\quad
k_{t}(x,x') =  k(x,x') -  k^{\TP}_{\Xb_t,x} (K_{\Xb_t,\Xb_t}+\tau \Ib)^{-1} k_{\Xb_t,x'},\label{GPt}
\end{eqnarray}
and $\sigma^2_{t}(x)= k_{t}(x,x)$, with $k_{\Xb_t,x}$ the $tB$ dimensional column vector with entries $[k_{\Xb_t,x}]_{(s-1)B+b}= k(x_{s,b},x)$, and $K_{\Xb_t,\Xb_t}$ the ${tB}\times{tB}$ positive definite covariance matrix with entries $[K_{\Xb_t,\Xb_t}]_{(s-1)B+b,(s'-1)B+b'} = k(x_{s,b}, x_{s',b'})$. 
We directly see from (\ref{GPt}) that accessing the posterior expressions require an $O((tB)^3)$ matrix inversion, which is a computational bottleneck for large values of $tB$. 

Note that in our problem formulation $f$ is fixed and observation noise has an unknown sub-Gaussian distribution. Using  a GP prior and assuming a Gaussian noise is merely for ease of modelling and does not affect our assumptions on $f$ and $\epsilon_{t,b}$. The notation  $\hat{f}$ is thus used to distinguish the GP model from the fixed $f$.




\subsection{Sparse Variational Gaussian Process Models with Inducing Points}\label{Sec:SVGP}

To overcome the cubic cost of exact GPs, SVGPs \citep{Titsias2009Variational,Hensman2013} instead approximate the GP posterior  
through a  set of \emph{inducing points} $\Zb_t = \{z_1,..., z_{m_t}\}$ ($z_i \in \X$, with $m_t << tB$). Conditioning on the \emph{inducing variables}
$\ub_t = \hat{f}(\Zb_t)$ (rather than the $tB$ observations in $\yb_t$) and specifying a prior Gaussian density $q_{t}(\ub_t) = \Nc(\mb_t, \Sb_t)$, yields an approximate posterior distribution that, crucially, is still  
a GP but with the significantly reduced computational complexity of $O(m_t^2t)$. The posterior mean and covariance of the SVGP are given in closed form as
\begin{eqnarray}\nn
{\mu}^{(s)}_t(x) = k_{\Zb_t,x}^{\TP} K^{-1}_{\Zb_t,\Zb_t}\mb_t\quad
{k}^{(s)}_t(x, x') = k(x,x')+ k_{\Zb_t,x}^{\TP}K^{-1}_{\Zb_t,\Zb_t} (\Sb_t - K_{\Zb_t,\Zb_t})K^{-1}_{\Zb_t,\Zb_t} k_{\Zb_t,x'}.\nn
\end{eqnarray}
The variational parameters $\mb_t$ and $\Sb_t$ are set as the maximizers of the evidence lower bound (ELBO, see Appendix~\ref{app:SVGP} for the details) and can be 
optimized numerically with mini-batching \citep{Hensman2013}. There are various standard ways in practice to select the locations of the inducing points $\Zb_t$, e.g. by using an experimental design, sampling from a k-DPP (that stands for determinantal point process),
or by optimizing them along with the inducing variables.

\subsection{Sparse Variational Gaussian Process Models with Inducing Features}\label{Subsec:feature}

An alternative approximation strategy is using inducing feature approximations \citep{hensman2017variational,Burt2019Rates,Dutordoir2020spherical}. Here, we
define inducing variables as the linear integral transform of $\hat{f}$ with respect to some \emph{inducing features} \citep{lazaro2009inter} $\psi_1(x),..,\psi_{m_t}(x)$, i.e we set our $i^{\textrm{th}}$ inducing variable as  $u_{t,i} = \int_{\X} \hat{f}(x)\psi_i(x)dx$.  Courtesy of Mercer's theorem, we can {usually} decompose our chosen kernel $k$ as the inner product of possibly infinite dimensional feature maps (see Theorem 4.1 in \cite{Kanagawa2018}) to provide the expansion $k(x,x') = \sum_{j=1}^{\infty}\lambda_j\phi_j(x).\phi_j(x')$ for eigenvalues $\{\lambda_j\in \Rr^+\}_{j=1}^{\infty}$ and eigenfunctions $\{\phi_j\in H_k\}_{j=1}^\infty$. If we set our inducing features to be the $m_t$ eigenfunctions with largest eigenvalues, it can be shown that  $\text{cov}(u_{t,i}, u_{t,j}) = \lambda_j \delta_{i,j}$ and $\text{cov}(u_{t,j}, \hat{f}(x)) = \lambda_j \phi_j(x)$ , yielding an approximate Gaussian Process model with  posterior mean and covariance given by
\begin{eqnarray}\nn
{\mu}^{(s)}_t(x) = \phib_{m_t}^{\TP}(x)\mb_t\qquad {k}^{(s)}_t(x,x') = k(x,x') +   \phib_{m_t}^{\TP}(x)(\Sb_t-\Lambda_{m_t})\phib_{m_t}(x').
\end{eqnarray}
Here, $\mb_t$ and $\Sb_t$ are inducing parameters (as above),  $\phib_{m_t}(x)\triangleq [\phi_1(x),...,\phi_{m_t}(x)]^{\TP}$ is the truncated feature vector and $\Lambda_{m_t}$ is the ${m_t}\times {m_t}$ diagonal matrix of eigenvalues, $[\Lambda_{m_t}]_{i,j}=\lambda_i\delta_{i,j}$. 


Inducing feature approximations have strong advantages, in particular a reduced computational cost and the fact that no inducing points need to be specified. However, accessing these eigenfeatures
require the Mercer decomposition of the used kernel, which is available for certain kernels on manifolds \citep{Borovitskiy2020,Dutordoir2020spherical}, but limited to low dimensions for others \citep{zhu1997gaussian,solin2020hilbert}.

\section{Scalable Thompson Sampling using Gaussian Process Models (S-GP-TS)}\label{Sec:Alg}

At each BO step $t$, GP-TS proceeds by drawing $B$ i.i.d.~samples $\{\hat{f}_{t,b}\}_{b\in [B]}$  from the posterior distribution of $\hat{f}$ and finding their maximizers, i.e. we select samples $x_{t,b}$ satisfying
\begin{eqnarray}
    \{x_{t,b} = \argmax_{x\in\X}\hat{f}_{t,b}(x)\}_{b\in [B]} \label{eq:maxsample}.
\end{eqnarray} However, since $\hat{f}_{t,b}$ is an infinite dimensional object, generating such samples is computationally challenging. Consequently, it is common to resort to approximate strategies, the most simple of which is to sample across an $N_t$ point discretization $D_t$ of $\X$ \citep{kandasamy2018parallelised} which can be obtained with an $O(N_t^3)$ cost (due to a required Cholesky decomposition). 

To improve the computational  efficiency of TS, a classical strategy \citep{hernandez2014predictive,Hernandez2014features} is to rely on kernel decompositions.
For instance, a sample $\hat{f}$ from a GP  can be expressed as a randomly weighted sum of the kernel's eigenfunctions  $
\hat{f}(x) =  \sum_{j=1}^\infty \sqrt{\lambda_j} w_j \phi_j(x),
$ or, in the case of shift-invariant kernels, the kernel's Fourier features $\psi_j(x)$ (see \cite{bochner1959lectures}) as $
\hat{f}(x) =  \sum_{j=1}^\infty w_j \psi_j(x)$. By truncating these infinite expansions to contain only the $M$ eigenfunctions with largest eigenvalues or $M$ random Fourier features, we have access to approximate but analytically tractable samples. 
For both expansions, the weights $w_j$ are sampled independently from a standard normal distribution. Conditioned on current $tB$ observations, the posterior distribution of $w_j$ are Gaussian with mean and covariance functions that can be calculated with an $O(M^3)$ computations, resulting in an $O(M^3+BNM)$ cost to draw $B$ Thompson samples.

Fast approximation strategies described above avoid costly matrix operations and work best only when sampling from GP priors. Posterior GP distributions are often too complex to be well-approximated by a finite feature representation \citep{wang2018batched,Mutny2018SGPTS,Calandriello2019Adaptive}. 
The recent work of \cite{wilson2020efficientsampling} tackled this issue by using truncated feature representations only to approximate the prior GP and a separate model update term to approximate posterior samples. For SVGP models, this has been shown to yield more accurate Thompson samples whilst incurring only an $O((m_t+M)BN)$, on top of the $O(tBm_t^2)$ SVGP model fit, per optimization step $t$.


For our theoretical analysis, we consider two distinct decoupled sampling rules inspired by \cite{wilson2020efficientsampling} , one for each of the two SVGP formulations presented above \citep[see][ for derivations and similar expressions for Fourier decompositions]{wilson2020efficientsampling}. The first rule is referred to as \emph{Decoupled Sampling with Inducing Points} and is defined as
\begin{eqnarray}\label{SampIndPoints}
\tilde{f}_{t}(x)
=\sum_{j=1}^M \alpha_t\sqrt{\lambda_j}w_j\phi_j(x) + \sum_{j=1}^{m_t}  v_{t,j} k(x,z_j),
\end{eqnarray}
where we have coefficients $v_{t,j} = [K^{-1}_{\Zb_t, \Zb_t}(\alpha_t(\ub_t - \mb_t) +\mb_t - \alpha_t\Phib_{m_t,M} \Lambda_M^{\frac{1}{2}}\wb_M)]_j$ for $\Phib_{m_t,M} = [\phib_M(z_1),..., \phib_M(z_{m_t})]^{\TP}$ and $\wb_{M} = [w_1,...,w_{M}]^{\TP}$. The weights $w_i$ are drawn i.i.d from $\Nc(0,1)$. 
(\ref{SampIndPoints})  is a modification of the sampling rule of \cite{wilson2020efficientsampling} where we have added a  scaling parameter $\alpha_t\in \Rr$ (with $\alpha_t=1$, the sampling rule of~\cite{wilson2020efficientsampling} is recovered).
When set to be greater than one, $\alpha_t$ serves to increases the variability of the approximate function samples (without changing their mean) and is used in our analysis to ensure sufficient exploration.

To efficiently sample from our second class of SVGP models, we also consider \emph{Decoupled Sampling with Inducing Features}: 
\begin{eqnarray}\label{SampIndVariables}
\tilde{f}_{t}(x)
=
\sum_{j=1}^M \alpha_t\sqrt{\lambda_j}w_j\phi_j(x) + \sum_{j=1}^{m_t} v_{t,j} {\lambda_j}\phi_j(x),
\end{eqnarray}
where $v_{t,j} = [\Lambda_{m_t}^{-1}(\alpha_t(\ub_t -\mb_t)+\mb_t -\alpha_t\Lambda_{m_t}^{\frac{1}{2}} \wb_{m_t})]_j$ for $\Lambda_{m_t}$ defined in Section \ref{Subsec:feature}.


\section{Regret Analysis of S-GP-TS}\label{Sec:Analysis}
Here, we first establish an upper bound on the regret of any approximate GP model (Theorem~\ref{The:RegretBoundsApproximate}) based on the quality of their approximate posterior, as parameterized in Assumptions~\ref{Ass1} and~\ref{Ass2}. We then discuss the consequences of Theorem~\ref{The:RegretBoundsApproximate} for the regret bounds and the computational complexity of S-GP-TS methods based on SVGPs and the decoupled sampling rules~\eqref{SampIndPoints} and~\eqref{SampIndVariables}.


\subsection{Regret Bounds Based on the Quality of Approximations}\label{Sec:RegSec}

Consider a TS algorithm using an approximate GP model. In particular, assume an approximate model is provided where $\tilde{k}_t$, $\tilde{\sigma}_t$ and $\tilde{\mu}_t$ are approximations of ${k}_t$, ${\sigma}_t$ and ${\mu}_t$, respectively. At each time $t$, a batch of $B$ samples $\{\tilde{f}_{t,b}\}_{b=1}^{B}$ is drawn from a GP with mean $\tilde{\mu}_{t-1}$ and the scaled covariance $\alpha_t^2\tilde{k}_{t-1}$. The batch of observation points $\{x_{t,b}\}_{b=1}^B$ are selected as the maximizers of $\{\tilde{f}_{t,b}\}_{b=1}^B$ over a discretization $D_t$ of the search space.

We start our analysis by making two assumptions on the \emph{quality} of approximations $\tilde{\mu}_t$, $\tilde{\sigma}_t$ of the posterior mean and the standard deviation. This parameterization is agnostic to the particular sampling rule (governing $\tilde{\mu}_t$ and $\tilde{\sigma}_t$)
and
provides valuable intuition that can be applied to any approximate method. When it comes to S-GP-TS (as the model governing $\tilde{\mu}_t$, $\tilde{\sigma}_t$), we show, in Sec.~\ref{Sec:prop1}, that these assumptions are satisfied under some conditions on the value of the parameters of the sampling rules.

\begin{assumption}[{quality} of the approximate standard deviation]\label{Ass1}
For the approximate $\tilde{\sigma}_{t}$, the exact $\sigma_t$, and for all $x\in \X$,
\[
\frac{1}{\underline{a}_t}\sigma_t(x)-\epsilon_t\le \tilde{\sigma}_{t}(x) \le \bar{a}_t\sigma_t(x)+\epsilon_t,
\]
where $1\le\underline{a}_t\le\underline{a}$, $1\le\bar{a}_t\le\bar{a}$ for all $t\ge1$ and some constants $\underline{a}, \bar{a}\in\Rr$, and $0\le\epsilon_t\le \epsilon$ for all $t\ge1$ and some small constant $\epsilon\in \Rr$.
\end{assumption}

\begin{assumption}[{quality} of the approximate prediction]\label{Ass2}
For the approximate $\tilde{\mu}_t$, the exact $\mu_t$ and $\sigma_t$, and for all $x\in \X$,
\[
|\tilde{\mu}_t(x) - \mu_t(x)| \le c_t {\sigma}_t(x),
\]
where $0\le c_t\le c$ for all $t\ge1$ and some constant $c\in \Rr$.
\end{assumption}
The following Lemma establishes a concentration inequality for the approximate statistics using the one for exact statistics~\cite[][Theorem $2$]{Chowdhury2017bandit}.

\begin{lemma}\label{Lemma:ConIneqApprox}
Under Assumptions~\ref{AssNoise},~\ref{AssNorm},~\ref{Ass1}~and~\ref{Ass2}, with probability at least $1-\delta$,
$
|f(x)-\tilde{\mu}_t(x)| \le  \u_t (\tilde{\sigma}_t(x)+\epsilon_t)
$, where $\u_t(\delta)=\underline{a}_t\left( \Bc+R\sqrt{2(\gamma_{tB}+1+\log(1/\delta))} +c_t\right)$.
\end{lemma}
Proof is provided in Appendix~\ref{app:Proofs}. Here, $\gamma_s$ is the \textit{maximal information gain}:
    $\gamma_s = \max_{A\subset\X, |A|= s}\Ic([y(x)]_{x\in A}; [\hat{f}(x)]_{x\in A})$,
where $\Ic([y(x)]_{x\in A}; [\hat{f}(x)]_{x\in A})$
denotes the mutual information~\citep[][Chapter $2$]{cover1999elements} between observations and the underlying GP model. The maximal information gain can itself be bounded for a specific kernel (see Sec. \ref{Sec:appl}).

Following \cite{srinivas2010gaussian} and \cite{Chowdhury2017bandit}, we consider a discretization $D_t$ of the search space satisfying the following assumption. 

\begin{assumption}\label{AssDisc}
The discretization $D_t$ is designed in a way that  $|f(x)-f(\mathtt{x}^{(t)})|\le 1/t^2$ for all $x\in\X$, where $\mathtt{x}^{(t)}=\argmin_{x'\in D_t}||x-x'||$ is the closest point (in Euclidean norm) to $x$ in $D_t$. 
The size of this discretization satisfies $|D_t|=N_t\le C(d,B) t^{2d}$ where $C(d,B)$ is independent of $t$ (\cite{Chowdhury2017bandit,srinivas2010gaussian}). 
\end{assumption}

We are now in a position to present regret bounds based on the quality of GP approximations:
\begin{theorem}\label{The:RegretBoundsApproximate}
Consider S-GP-TS with $\alpha_t = 2\u_{t}(1/(t^2))$. Under Assumptions~\ref{AssNoise},~\ref{AssNorm},~\ref{Ass1}~,~\ref{Ass2} and~\ref{AssDisc}, the regret defined in~\eqref{eq:regdef}, satisfies
\begin{eqnarray}\nn
R(T,B;f) &\le& 
30\bar{a}\beta_TB\sqrt{\frac{2T\gamma_T}{\log(1+\frac{1}{\tau})}}+
(31\beta_T+\alpha_T)\epsilon TB + 15B\Bc+2B \\
&&\hspace{-5em} ={O}\left( \underline{a}\bar{a}BR\sqrt{d\gamma_T(\gamma_{TB}+\log(T)) T\log(T)} + \underline{a}\epsilon TBR\sqrt{d(\gamma_{TB}+\log(T))\log(T)}  \right),
\end{eqnarray}
where $\beta_t = \alpha_t(b_t+\frac{1}{2})$ with $b_t = \sqrt{2\log(N_t t^2)}$. 
\end{theorem}


See the proof in Appendix~\ref{app:Proofs}. This regret bound scales with the product of the ratios $\underline{a}$ and $\bar{a}$, with an additive term depending on the additive approximation error in the standard deviation.

\subsection{Approximation Quality of the Decoupled Sampling Rule}\label{Sec:prop1}

For S-GP-TS with inducing points, we assume, as in \cite{Burt2019Rates},  that the inducing points are sampled according to a discrete k-DPP. While this might be costly in practice, 
\cite{Burt2019Rates} showed that $\Zb_t$ can be efficiently sampled from \emph{$\epsilon_0$ close} sampling methods without compromising the predictive quality of SVGP. For both sampling rules, we also assume in our analysis that the Mercer decomposition of the kernel is used.

The quality of the approximation can be characterized using the spectral properties of the GP kernel. Let us define the tail mass of eigenvalues $\delta_M =\sum_{i=M+1}^\infty\lambda_i\bar{\phi}^{2}_i$ where $\bar{\phi}_i =\max_{x\in\X} \phi_i(x)$.
With decaying eigenvalues, including sufficient eigenfunctions in the feature representation results in a small $\delta_M$. 
In addition, \cite{Burt2019Rates} showed that, for an SVGP, a sufficient number of inducing variables ensures that the Kullback–Leibler (KL) divergence between the approximate and the true posterior distributions diminishes. Consequently, the approximate posterior mean and the approximate posterior variance converge to the true ones. Building on this result, we are able to prove Proposition~\ref{Prop:paramters} on the quality of approximations.

\begin{proposition}\label{Prop:paramters}
For S-GP-TS based on sampling rule~\eqref{SampIndPoints} with $\alpha_t=1$ and an SVGP using an $\epsilon_0$ close k-DPP for selecting $\Zb_t$, with probability at least $1-\delta$, Assumptions~\ref{Ass1} and~\ref{Ass2} hold with parameters $c_t =  \sqrt{\kappa_t}$, $\underline{a}_t = \frac{1}{\sqrt{1-\sqrt{3\kappa_t}}}$, $\bar{a}_t = \sqrt{1+\sqrt{3\kappa_t}}$, and $\epsilon_t = \sqrt{C_1m_t\delta_M}$,
where $C_1$ is a constant specified in the appendix and $\kappa_t = \frac{2tB(m_t+1)\delta_{m_t}}{\tau\delta} + \frac{4tB\epsilon_0}{\tau\delta}$.

For S-GP-TS based on sampling rule~\eqref{SampIndVariables} with $\alpha_t=1$, Assumptions~\ref{Ass1} and~\ref{Ass2} hold with parameters $c_t =  \sqrt{\kappa_t}$, $\underline{a}_t = \frac{1}{\sqrt{1-\sqrt{3\kappa_t}}}$, $\bar{a}_t = \sqrt{1+\sqrt{3\kappa_t}}$, and $\epsilon_t = \sqrt{C_1m_t\delta_M}$,
where $C_1$ is the same constant as above and $\kappa_t = \frac{2tB\delta_{m_t}}{\tau}$.
\end{proposition}

Note that our proposition requires extending the results of~\cite{Burt2019Rates} in two non-trivial ways. First, the decoupled sampling rules introduce an additional error. 
Secondly, \cite{Burt2019Rates} built their convergence results on the assumption that the observation points $x_{t,b}$ are drawn from a prefixed distribution, which is not the case in S-GP-TS, where $x_{t,b}$ are selected according to an experimental design method. A detailed proof of Proposition~\ref{Prop:paramters} is provided in Appendix~\ref{app:Proofs}. 


\subsection{Application of Regret Bounds to Mat{\'e}rn and SE Kernels}\label{Sec:appl}

We now investigate the application of Theorem~\ref{The:RegretBoundsApproximate} to the Squared Exponential (SE) and Mat{\'e}rn kernels, widely used in practice \citep[see, e.g.,][]{Rasmussen2006, Snoek2012practicalBO}. 
%
%
%
In the case of a Mat{\'e}rn kernel with smoothness parameter $\nu>\frac{d}{2}$ it is known that $\lambda_j  =O(j^{-\frac{2\nu+d}{d}})$~\citep{MaternEigenvaluessantin2016}. For the SE kernel, we have $\lambda_j  = O(\exp(-j^{\frac{1}{d}}))$~\citep{SEEigenvalues,Gabriel2020practicalfeature}. 
With these bounds on the spectrum of the kernels and the specific bounds on the maximal information gain \citep[e.g., $\gamma_s\le O(\log(s)^{d+1})$ for SE and $\gamma_s\le O(s^{d/(2\nu+d)}\log(s))$ for Mat{\'e}rn,][]{vakili2021information},
Theorem~\ref{The:RegretBoundsApproximate} and Proposition~\ref{Prop:paramters} result in the following theorem.
\begin{theorem}\label{The:MatSE}
Under Assumptions~\ref{AssNoise} and \ref{AssNorm}, with the algorithmic parameters, kernels and sampling rules specified in Table \ref{Table:CC}, S-GP-TS offers
$
R(T,B;f) = O(B\sqrt{\gamma_T\gamma_{TB}T\log(T)}).
$

\end{theorem}


With a batch size $B=1$ Theorem \ref{The:MatSE} recovers the same regret bounds as the exact GP-TS~\citep{Chowdhury2017bandit}. 
{We also note that for a fair comparison in terms of both the batch size and number of samples we should consider $T'=TB$
as the number of samples. In that case, our regret bound becomes $O(\sqrt{B\gamma_{T'/B}\gamma_{T'} T'\log(T'/B)})$, which scales at most with $\sqrt{B}$. That is $\sqrt{B}$ tighter than the trivial scaling with $B$.}

In order to prove Theorem~\ref{The:MatSE}, the algorithmic parameters $M$ and $m_t$ must be selected large enough such that approximation parameters $\underline{a}, \bar{a}, c, \epsilon$ in Assumptions~\ref{Ass1} and~\ref{Ass2} are sufficiently small. Using the relation between the algorithmic parameters, the approximation parameters and $m_t$ provided by Proposition~\ref{Prop:paramters}, the regret bound follows from Theorem~\ref{The:RegretBoundsApproximate}. See Appendix~\ref{app:Proofs} for a detailed proof.

The values of $M$ and $m_t$ required for Theorem \ref{The:MatSE} are summarized in Table~\ref{Table:CC}. We also show the resulting computational cost of each sampling rule (as given by $O\left(B(M+m_T)N_T T + Bm_T^2T^2\right)) $, explicitly demonstrating the improvement of S-GP-TS over the $O(BN_T^3T+B^3T^4)$ computational cost of the vanilla GP-TS. Note that, for the Mat{\'e}rn kernel under sampling rule~\eqref{SampIndPoints}, $\nu$ is required to be sufficiently larger than $\frac{d}{2}$ in order for $m_t$ to grow slower than $t$.
\begin{table*}[h]
\renewcommand{\tabcolsep}{2pt}
\renewcommand{\arraystretch}{1.5}

\centering
\begin{tabular}{c c c c}
 \Xhline{2\arrayrulewidth}
    &  & Inducing points~\eqref{SampIndPoints} & Inducing features~\eqref{SampIndVariables}\\
\multirow{2}{*}{\rotatebox[origin=c]{90}{Mat{\'e}rn}} & Condition & $m_t \sim T^{\frac{2d}{2\nu-d}}$, $M\sim T^{\frac{(2\nu+d)d}{2(2\nu-d)\nu}}$ & $m_t \sim T^{\frac{d}{2\nu}}$, $M\sim T^{\frac{(2\nu+d)d}{4\nu^2}}$ \\ 
      & Cost & \small{$O\left(BN_TT^{\frac{4\nu^2+d^2}{2(2\nu-d)\nu}}+
 BT^2\min\{T^{\frac{4d}{2\nu-d}}, T^2\} \right)$}& 
    \small{$O\left( BN_TT^{\frac{(2\nu+d)^2-2\nu d}{4\nu^2}} + BT^{\frac{2\nu+d}{\nu}}  \right)$} \\ 
  \hline
\multirow{2}{*}{\rotatebox[origin=c]{90}{SE}} &  Condition & $m_t,M\sim (\log(T))^{d}$ & $m_t,M\sim (\log(T))^{d}$ \\
 & Cost &
 \small{$O\left(BN_TT\log^{d}(T)+BT^2\log^{2d}(T)\right)$}& \small{$O\left(BN_TT\log^{d}(T)+BT^2\log^{2d}(T)\right)$}
\end{tabular}
\caption{Conditions on the number of features $m_t$ and inducing variables $M_T$ required for Theorem \ref{The:MatSE}, alongisde the resulting cost of each decoupled sampling method.}\label{Table:CC}
\end{table*}


\section{Experiments}\label{Sec:Exp}

We now provide an empirical evaluation of S-GP-TS. As \cite{wilson2020efficientsampling} have already comprehensively demonstrated the practical advantage of decoupled sampling for problems with small optimization budgets, we focus here on scalability of S-GP-TS, and in particular a) its efficiency with large batch size, b) its ability to handle large data volumes. 
We first investigate a collection of classical synthetic problems for BO, before demonstrating S-GP-TS in a challenging real-world high-throughput molecular design considered by \cite{hernandez2017parallel}. Our synthetic experiments focus on multi-modal problems with substantial observation noise, as these cannot be solved accurately with a small budget yet are still unsuitable for local, exhaustive, or deterministic optimization routines. Our implementation is provided as part of the open-source toolbox \texttt{trieste} \citep{trieste} \footnote{\url{https://github.com/secondmind-labs/trieste}} and relies also on  \texttt{gpflow} \citep{matthews2017gpflow} and \texttt{gpflux} \citep{dutordoir2021gpflux}.

{As is often the case, our regret-based analysis applies to a version of the algorithm that is slightly different to a practically viable BO method. Rather than focusing on recreating our algorithm exactly in the very limited settings (e..g a 1-d RBF kernel for which we can calculate eigen-features exactly) that are of little interest to the BO community, we instead choose to demonstrate the practical strength and unprecedented scalability of S-GP-TS by investigating an implementation that could be used by practitioners. The resulting algorithms demonstrated in this section are still well-aligned with our work through their use of sparse GP surrogate models and decoupled Thompson sampling.}


\subsection{Synthetic Benchmarks}\label{Sec:Toy}
We first consider two toy problems: Hartmann (6 dim, moderately multi-modal) with a large additive noise and Shekel (4 dim, highly multi-modal) with moderate noise,
see Appendix~\ref{app:Experiments} for the full description.
Our SVGP models use inducing points and a Mat{\'e}rn kernel with smoothness parameter $\nu=2.5$.
As eigenfunctions for this kernel are limited to small dimensions \citep{solin2020hilbert}, 
we implement decoupled TS using the easily accessible random Fourier Features (RFF). Note that \cite{wilson2020efficientsampling} have shown decoupled sampling to significantly alleviate the \emph{variance starvation} phenomenon (underestimating the variance of points far from the observations~\citep{wang2018batched,Mutny2018SGPTS}) that typically hampers the efficacy of RFFs. We use $M=1000$ features and maximise each sample as in (\ref{eq:maxsample}) using L-BFGS-B, starting from the best point among a large sample. 

{As sampling inducing points from a k-DPP is prohibitively costly} for the repeated model fitting required by BO loops, we use the the greedy variance selection method of \cite{Burt2019Rates} which is \emph{$\epsilon_0$ close} to k-DPP and has been shown to outperform optimisation of inducing points in practice. We also consider the practical alternative of choosing inducing points chosen by a k-means clustering of the observations. As the optimisation progresses, observations are likely to be concentrated in the optimal regions, so clustering would result in somehow ``targeted'' inducing points for BO.  In order to control the computational overhead of S-GP-TS {and to allow an efficient computational implementation (i.e. avoiding Tensorflow recompilation issues)}, we use a fixed number $m_t$ of points, set to either 250 or 500. Similarly, we set the covariance scaling parameter $\alpha_t=1$ {to avoid having dynamic tunable parameters, like those that plague UCB-based approaches}.


For each experiment, we run $t=50$ steps of S-GP-TS with $B=100$ (i.e. 5,000 total observations). 
For baselines, we compare against $t=750$ steps of standard sequential non-batch BO routines with an exact GP model:
Expected Improvement \citep[EI,][]{jones1998efficient}, 
Augmented Expected Improvement \citep[AEI,][]{huang2006global}, and an extension of Max-value Entropy search suitable for noisy observations
\citep[GIBBON,][]{moss2021gibbon}.
Due to the large number of steps, we only consider low-cost but high-performance acquisitions, following the cost-benefit analysis of \cite{moss2021gibbon}, and exclude the popular knowledge gradient \citep{wu2016parallel} or classical entropy search \citep{hennig2012entropy,hernandez2014predictive}. Popular existing batch acquisition functions do not scale to batches as large as $B=100$, however, we present their performance on smaller batches across additional experiments in Appendix~\ref{app:Experiments}. We report simple regret of the current believed best solution (maximizer of the current model mean) across the previously queried data points. All results are averaged over $30$ runs and reported as a function of either the number of function evaluations ($tB$ for S-GP-TS and $t$ for the baselines), or the number of BO iterations, in Figure \ref{fig:synthetic}. 
   \vspace{-2mm}
\begin{figure*}[!ht]
 \centering
   \includegraphics[trim=0mm 3mm 0mm 2mm, clip, height=50mm]{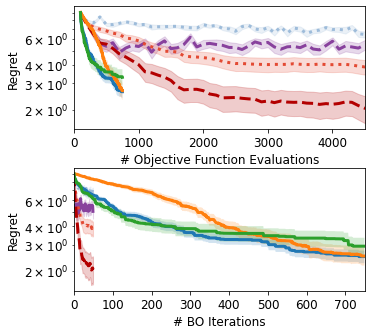}
   \includegraphics[trim=0mm 3mm 0mm 2mm, clip, height=50mm]{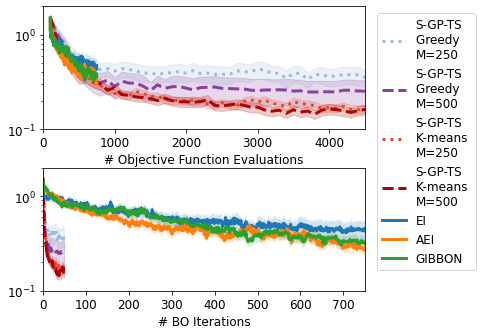}
      \vspace{-2mm}
   \caption{Simple regret on Shekel (4D, left) and Hartmann (6D, right). 
    When considering regret with respect to the total number of objective function evaluations $tB$ (top panels, purely sequential setting),  all S-GP-TS methods are initially less efficient (Shekel) or match the performance (Hartmann) of the best baselines, however the best S-GP-TS approach is able to efficiently allocate its additional budget to achieve lower final regret. When considering regret with respect to the BO iteration (bottom panels, idealised parallel setting), S-GP-TS achieves low regret in a fraction of the iterations required by the standard BO routines.
   }\label{fig:synthetic}
   \vspace{-2mm}
\end{figure*}

The fact that S-GP-TS is able to find solutions on both benchmarks with substantially improved regret than found by standard BO, provides strong evidence that S-GP-TS is effectively leveraging parallel resources. Moreover, as these higher-quality solutions were only found after large number of total evaluations, Figure \ref{fig:synthetic} also highlights the necessity for BO routines, like S-GP-TS, that can handle these larger (heavily parallelized) optimization budgets. {We reiterate that the existing BO baselines cannot handle as many evaluations as S-GP-TS, becoming prohibitively slow once we surpass $750$ data-points)}. When considering the regret achieved per individual function evaluation, we typically expect batch routines to be less efficient than purely sequential BO routines. However, in the case of the Hartmann function (the benchmark with the largest observation noise), we see that our best S-GP-TS exactly matches (before going on to exceed) the performance of the sequential routines, suggesting that S-GP-TS is a particularly effective optimizer for functions with significant levels of observation noise. 

Note that the performance of S-GP-TS is sensitive to its chosen inducing points, with k-means providing the most effective routines. On Hartmann,  250 inducing points is sufficient to deliver good performances, while on Shekel, which is much more multimodal, using a larger number is critical. 

\subsection{High-throughput Molecular Search}\label{Sec:Molecule}

Finally, we investigate the performance of S-GP-TS with respect to an established baseline for high-throughput molecular screening. Although molecular search has been tackled many times with BO \citep{gomez2016design,griffiths2017constrained,moss2020boss}, only the approach of \cite{hernandez2017parallel} - standard (non-decoupled) TS over a Bayesian neural network (BNN-TS) - is truly scalable. We now recreate the largest experiment considered by \cite{hernandez2017parallel},
where the objective is to uncover the top $10\%$ of molecules in terms of power conversion efficiency among a library of 2.3 million candidate 
from the Harvard Clean Energy Project \citep{hachmann2011harvard}. Molecules are encoded as Morgan circular fingerprints of Bond radius 3 (i.e. 512-dimensional bit vectors, see \cite{rogers2010extended}).

As the standard GP kernels considered above are not suitable for sparse and high-dimensional molecule inputs \citep{moss2020gaussian}, we instead build our SVGP with a zeroth order ArcCosine kernel \citep{cho2012kernel}, chosen due to its strong empirical performance under sparsity and as it permits a random decomposition that can be exploited to perform decoupled TS. In particular, we use the $M$-feature decomposition investigated by \cite{cutajar2017random} of
\begin{align*}
k_{arc}(\textbf{x},\textbf{x}')=2\int d\textbf{w}\frac{\textrm{e}^{-\frac{\|\textbf{w}\|^2}{2}}}{(2\pi)^{d/2}}\Theta(\textbf{w}^T\textbf{x})\Theta(\textbf{w}^T\textbf{x}')\approx\frac{2}{M}\sum_{j=1}^M\Theta(\textbf{w}_j^T\textbf{x})\Theta(\textbf{w}_j^T\textbf{x}'),
\end{align*}where $\Theta(.)$ is the Heaviside step function and $\textbf{w}_j\sim\mathcal{N}(0,I)$. 

In our experiments, we use $M=1\,000$ random features and, to avoid memory issues, we compute our GP samples over a random subset of $100\,000$ molecules (renewed at each sample). We run S-GP-TS twice, once with $m_t=500$ and once with $\,2000$ inducing points. We chose inducing points as uniform samples from the already evaluated molecules (for each model step), as preliminary experiments showed that neither the k-means nor greedy selection routines discussed above were effective when applied to sparse and high-dimensional molecular fingerprint inputs. 

Following \cite{hernandez2017parallel}, we report the recall (fraction of the top $10\%$ of molecules so far chosen by the BO loop) for S-GP-TS, along with the performance of BNN-TS, a greedy BNN  (that queries the $B$ maximizers of the BNN's posterior mean), and a random search baseline (all taken from \cite{hernandez2017parallel}). All routines (including our S-GP-TS) are ran for $t=250$ successive batches of $B=500$ molecules. Figure \ref{fig:molecule} shows that S-GP-TS is able to perform effective batch optimization over very large optimization budgets (120,000 total evaluations) and, when using  $m=2000$ or even just $m=500$ inducing points, S-GP-TS matches the performance of \cite{hernandez2017parallel}'s BNN-based TS and greedy sampling approaches, respectively. Note that due to the high computational demands of this experiment, we report just a single replication of S-GP-TS (a limitation also of \cite{hernandez2017parallel}'s results). However, we stress that an additional realization of the $m=500$ experiment returned indistinguishable results.

\begin{figure*}[!ht]
 \centering
  \includegraphics[trim=0mm 3mm 0mm 2mm, clip, width=.6\textwidth]{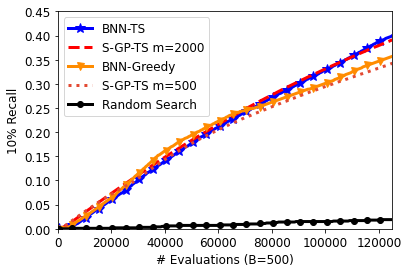}
     \vspace{-2mm}
   \caption{Proportion of the top 10\% of molecules found by each of the search routines. S-GP-TS is able to process substantial data volumes and effectively allocates large batches, matching the performance of the well-established BNN baselines.  }\label{fig:molecule}
      \vspace{-5mm}
\end{figure*}

\section{Discussion}
\label{sec:discussion}

We have shown that S-GP-TS enjoys the same regret order as exact GP-TS but with a greatly reduced $O(N_t M)$ computation per step $t$, compared to the $O(N_t^3)$ cost of the standard sampling. 
However, the discretization size $N_t$ is exponential in the dimension $d$ of the search space and so remains a limiting computational factor when optimizing over high dimensional search spaces. Hence, while S-GP-TS with decoupled sampling rule allows orders of magnitude larger optimization budgets compared to vanilla GP-TS, it still suffers from the \emph{curse of dimensionality}. Intuitively, this seems inevitable due to NP-Hardness of non-convex optimization problems \citep[see, e.g.,][]{Jian2017NPHard} as required to find the maximizer of the GP-UCB acquisition function \citep[see, e.g.,][]{Calandriello2019Adaptive}, or even in the application of UCB to linear bandits \citep{Dani2008}. In particular, the computational cost of the \textit{state-of-the-art} adaptive sketching method for implementing GP-UCB \cite{Calandriello2019Adaptive} was reported as $O(N_Td^2_{\text{eff}})$ where $d_{\text{eff}}$, referred to as the effective dimension of the problem, is upper bounded by $\gamma_T$.

An important practical consideration when using S-GP-TS in practice is how to choose its inducing points. The performance improvement provided by choosing inducing points by k-means rather than greedy variance selection, as demonstrated in our experiments, raises the possibility that BO-specific routines for choosing inducing points could allow even better performance. This is an important avenue for future work.






\bibliography{references.bib}
\bibliographystyle{unsrt}

\newpage
\appendix


\section{Complements on SVGPs}\label{app:SVGP}
As discussed in Section \ref{Sec:GPs}, SVGPs approximate the posterior of exact GPs through either a set of inducing points  $\Zb_t \triangleq \{z_1,..., z_{m}\}$ or through a set of inducing features $\phib_{m}(x)\triangleq \{\phi_1(x),...,\phi_m(x)\}$. The resulting inducing features, defined as $u_{t,i} = \hat{f}(z_{t,i})$  (for inducing points) or $u_{t,i} = \int_{\X} \hat{f}(x)\phi_i(x)dx$ (for inducing features), are assumed to follow a prior Gaussian density $q_{t}(\ub_t) = \Nc(\mb_t, \Sb_t)$. We now discuss how to set these variational parameters $\mb_t$ and $\Sb_t$ for a given dataset.

For SVGPS, the posterior mean and covariance is given in closed form as
\begin{eqnarray}\nn
{\mu}^{(s)}_t(x) = k_{\Zb_t,x}^{\TP} K^{-1}_{\Zb_t,\Zb_t}\mb_t\quad
{k}^{(s)}_t(x, x') = k(x,x')+ k_{\Zb_t,x}^{\TP}K^{-1}_{\Zb_t,\Zb_t} (\Sb_t - K_{\Zb_t,\Zb_t})K^{-1}_{\Zb_t,\Zb_t} k_{\Zb_t,x'}\nn
\end{eqnarray}
and
\begin{eqnarray}\nn
{\mu}^{(s)}_t(x) = \phib_{m_t}^{\TP}(x)\mb_t\qquad {k}^{(s)}_t(x,x') = k(x,x') +   \phib_{m_t}^{\TP}(x)(\Sb_t-\Lambda_{m_t})\phib_{m_t}(x'),
\end{eqnarray}
for the inducing point and inducing features representations, respectively. See Section \ref{Sec:GPs} or \cite{Burt2019Rates} for more details. However, the marginal likelihood for these models is intractable, and so, as is common practice in variational inference methods, we set of our variational parameters (as-well as the SVGP's kernel parameters) to maximize instead the tractable Evidence-based Lower BOund (ELBO).

For inducing point SVGPS, the ELBO can be written as
\begin{eqnarray}\nn
\text{ELBO}(t) &=& -\frac{1}{2}\yb_t^{\TP} (Q_t+\tau \Ib_t)^{-1}\yb_t - \frac{1}{2}\log|Q_t+\tau \Ib_t| - \frac{t}{2}\log(2\pi)-\frac{\theta_t}{2\tau},\label{eq:elbo}
\end{eqnarray}
where $Q_t = K^{\TP}_{\Zb_t,\Xb_t}K^{-1}_{\Zb_t,\Zb_t}K_{\Zb_t,\Xb_t}$, $K_{\Zb_t,\Xb_t} = [k_{z_i, x_j}]_{i,j}$, $i=1,\dots, m_t$, $j=1,\dots,t$, $\Ib_t$ is the $t\times t$ identity matrix and $\theta_t = \text{Tr}(K_{\Xb_t,\Xb_t}- Q_t)$. See \cite{Hensman2013} for a full derivation.

For inducing feature SVGP, the expression of ELBO is the same but with $Q_t = K^{\TP}_{\phib_{m_t},\Xb_t} \Lambda^{-1}_{m_t}K_{\phib_{m_t},\Xb_t}$, $K_{\phib_{m_t},\Xb_t} = [\lambda_i\phi_i(x_j)]_{i,j}$, $i=1,\dots, m_t$, $j=1,\dots,t$. 

To optimize the ELBO in practice, \cite{Hensman2013} proposed a numerical solution allowing for mini-batching \citep[see also][]{Burt2019Rates} and the use of stochastic gradient descent algorithms such as Adam~\citep{kingma2014adam}. In addition, \cite{Titsias2009Variational} provides an explicit solution for the convex optimization problem of finding $(\mb_t, \Sb_t)$, allowing more involved alternate optimization schemes.

\section{Detailed Proofs}\label{app:Proofs}

In this section, we provide detailed proofs for Theorem~\ref{The:RegretBoundsApproximate}, Lemma~\ref{Lemma:ConIneqApprox}, Proposition~\ref{Prop:paramters} and Theorem~\ref{The:MatSE}, in order.
%
%

\subsection{Proof of Theorem~\ref{The:RegretBoundsApproximate}}

Before presenting the proof of Theorem~\ref{The:RegretBoundsApproximate}, we first overview the regret bound for vanilla GP-TS~\citep[][Theorem $4$]{Chowdhury2017bandit}.

\begin{proof}[The Existing Regret Bound for Vanilla GP-TS] 

\cite{Chowdhury2017bandit} proved that, with probability at least $1-\delta$, $|f(x) - \mu_{t}(x)|\le {u}_t\sigma_t(x)$, where ${u}_t = \left( B+R\sqrt{2(\gamma_{t}+1+\log(1/\delta))} \right)$ and $\gamma_t$ is the {maximal information gain}. Based on this concentration inequality, \cite{Chowdhury2017bandit} showed that the regret of GP-TS scales with the cumulative uncertainty at the observation points measured by the standard deviation: $O(\sum_{t=1}^T {u}_t\sigma_{t-1}(x_t))$.
Furthermore, \cite{srinivas2010gaussian} showed that $\sum_{i=1}^t\sigma_{i-1}^2(x_i)\le \gamma_t$. Using this result and applying Cauchy-Schwarz inequality to $O(\sum_{t=1}^T {u}_t\sigma_{t-1}(x_t))$,  \cite{Chowdhury2017bandit} proved that
$
R( T,1; f) = {O}\left(\gamma_T\sqrt{T\log(T)}\right)
$, for vanilla GP-TS. 

\end{proof}

We build on the analysis of GP-TS in~\cite{Chowdhury2017bandit} to prove the regret bounds for S-GP-TS. We stress that despite some similarities in the proof, the analysis of standard GP-TS does not extend to S-GP-TS. This proof characterizes the behavior of the upper bound on regret in terms of the approximation constants, namely $\underline{a}, \bar{a}, {c}$ and $\epsilon$. 
A notable difference is that the additive approximation error in the posterior standard deviation ($\epsilon_t$) can cause under-exploration which is an issue the analysis of exact GP-TS cannot address. In addition, we account for the effect of batch sampling on the regret bounds. 

We first focus on the instantaneous regret at each time $t$ within the discrete set, $f({\mathtt{x}^*}^{(t)}) - f(x_{t,b})$. Recall ${\mathtt{x}^*}^{(t)} \triangleq \text{argmin}_{x'\in D_t}||x^*-x'||$ from Assumption~\ref{AssDisc}. It is then easy to upper bound the cumulative regret by the cumulative value of $f({\mathtt{x}^*}^{(t)}) - f(x_{t,b})+\frac{1}{t^2}$ as our discretization ensures that $f(x^*)-f({\mathtt{x}^*}^{(t)})\le \frac{1}{t^2}$. 
For upper bounds on instantaneous regret, we start with concentration of GP samples $\tilde{f}_{t,b}$ around their predicted values and the concentration of the prediction around the true objective function. We then consider the anti-concentration around the optimum point. The necessary anti-concentration may fail due to approximation error in the standard deviation around the optimum point. We thus consider two cases of low and sufficiently high standard deviation at ${\mathtt{x}^*}^{(t)}$ separately. While a low standard deviation implies good prediction at ${\mathtt{x}^*}^{(t)}$, a sufficiently high standard deviation guarantees sufficient exploration. We use these results to upper bound the instantaneous regret at each time $t$ with uncertainties measured by the standard deviation. 

\textbf{Concentration events $\Ec_t$ and $\tilde{\Ec}_t$:}

\textbf{Define $\Ec_t$} as the event that at time $t$, for all $x\in D_t$, 
$
|f(x) - \tilde{\mu}_{t-1}(x)| \le \frac{1}{2}\alpha_t(\tilde{\sigma}_{t-1}(x)+\epsilon_t)
$.
Recall $\alpha_t = 2\u_t(1/(t^2))$. Applying lemma~\ref{Lemma:ConIneqApprox}, we have $\Pr[\Ec_t]\ge 1-\frac{1}{t^2}$.

\textbf{Define $\tilde{\Ec}_t$} as the event that for all $x\in D_t$, and for all $b\in [B]$,
$
|\tilde{f}_{t,b}(x) - \tilde{\mu}_{t-1}(x)| \le  \alpha_tb_t \tilde{\sigma}_{t-1}(x)
$
where $b_t = \sqrt{2\ln(B N_t  t^2)}$. We have 
$
\Pr[\tilde{\Ec}_t] \ge 1-\frac{1}{t^2}.
$

\emph{Proof.}
For a fixed $x\in D_t$, and a fixed $b\in [B]$,
\begin{eqnarray}\nn
\Pr\left[|\tilde{f}_{t,b}(x) - \tilde{\mu}_{t-1}(x)| >  \alpha_{t}b_t \tilde{\sigma}_{t-1}(x)\right] <\exp(-\frac{b_t^2}{2}) = \frac{1}{BN_tt^2}.
\end{eqnarray}
The inequality holds because of the following bound on the CDF of a normal random variable $1-\text{CDF}_{\Nc(0,1)}(c)\le \frac{1}{2}\exp(-\frac{c^2}{2})$ and the observation that $\frac{\tilde{f}_{t,b}(x) - \tilde{\mu}_{t-1}(x)}{\alpha_t \tilde{\sigma}_{t-1}(x)}$ has a normal distribution. Applying a union bound we get
$
\Pr[\bar{\tilde{\Ec}}_t] \le \frac{1}{t^2}
$
which gives us the bound on probability of $\tilde{\Ec}_t$.  ~~~~~ $\square$

We thus proved $\Ec_t$ and $\tilde{\Ec}_t$ are high probability events. This will facilitate the proof by conditioning on $\Ec_t$ and $\tilde{\Ec}_t$. 
Also notice that when both $\Ec_t$ and $\tilde{\Ec}_t$ hold true, we have, for all $x\in D_t$, and for all $b\in[B]$,
\begin{eqnarray}\label{eq:beta}
|\tilde{f}_{t,b}(x) - f(x)| \le \beta_t\tilde{\sigma}_{t-1}(x) + \frac{1}{2}\alpha_t\epsilon_t
\end{eqnarray}
where $\beta_t = \alpha_t(b_t+\frac{1}{2})$.

\paragraph{Anti Concentration Bounds.} It is standard in the analysis of TS methods to prove sufficient exploration using an anti-concentration bound. 
That establishes a lower bound on the probability of a sample being sufficiently large (so that the corresponding point is likely to be selected by TS rule).
For this purpose, we use the following bound on the CDF of a normal distribution: $1-\text{CDF}_{\Nc(0,1)}(c)\ge \frac{\exp(-{c^2})}{4c\sqrt{\pi}}$. The underestimation of the posterior standard deviation at the optimum point however might result in an under exploration. On the other hand, a low standard deviation at the optimum point implies a low prediction error. We use this observation in our regret analysis by considering the two cases separately. Specifically, the regret $ f({\mathtt{x}^*}^{(t)}) - f(x_{t,b})$ at each time $t$ for each sample $b$ is bounded differently under the conditions: I. $\tilde{\sigma}_{t-1}({\mathtt{x}^*}^{(t)})\le \epsilon_t$ and II. $\tilde{\sigma}_{t-1}({\mathtt{x}^*}^{(t)})> \epsilon_t$. \\

\textbf{Under Condition} I ($\tilde{\sigma}_{t-1}({\mathtt{x}^*}^{(t)})\le \epsilon_t$), when both $\Ec_t$ and $\tilde{\Ec}_t$ hold true, we have
\begin{eqnarray}\nn
&&\hspace{-4em}f({\mathtt{x}^*}^{(t)}) - f(x_{t,b})\\\nn
&\le& \tilde{f}_{t,b}({\mathtt{x}^*}^{(t)}) + \beta_t\tilde{\sigma}_{t-1}({\mathtt{x}^*}^{(t)})+ \frac{1}{2}\alpha_t\epsilon_t \\\nn
&&~~~~~- \tilde{f}_{t}(x_{t,b}) + \beta_t\tilde{\sigma}_{t-1}(x_{t,b})+ \frac{1}{2}\alpha_t\epsilon_t~\hspace{7em}\text{by~\eqref{eq:beta},}\\\label{eq:boundstarsig}
&\le&  \beta_t\tilde{\sigma}_{t-1}({\mathtt{x}^*}^{(t)})+ \beta_t\tilde{\sigma}_{t-1}(x_{t,b})+\alpha_t\epsilon_t~\hspace{1.2em}\text{by~the selection rule of TS,}\\\nn
&\le&\beta_t\tilde{\sigma}_{t-1}(x_{t,b}) +    (\beta_t +\alpha_t)\epsilon_t~\hspace{9em}~~\text{by Condition I}.
\end{eqnarray}
that upper bounds the instantaneous regret at time $t$ by a factor of approximate standard deviation up to an additive term caused by approximation error. Since $f({\mathtt{x}^*}^{(t)}) - f(x_{t,b})\le 2B$, under Condition I,
\begin{eqnarray}\label{Ass1simplereg}
\E[f({\mathtt{x}^*}^{(t)}) - f(x_{t,b})] \le  \beta_t\tilde{\sigma}_{t-1}(x_{t,b}) + (\beta_t+\alpha_t)\epsilon_t +\frac{4B}{t^2}.
\end{eqnarray}
where the inequality holds by $\Pr[\bar{{\Ec}}_t ~\text{or}~\bar{\tilde{\Ec}}_t  ]\le \frac{2}{t^2}$.

\textbf{Under Condition} II ($\tilde{\sigma}_{t-1}({\mathtt{x}^*}^{(t)})> \epsilon_t$), we can show sufficient exploration by anti-concentration at the optimum point. In particular under Condition II, if $\Ec_t$ holds true, we have \begin{eqnarray}\label{eq:plow}
\Pr[\tilde{f}_{t,b}({\mathtt{x}^*}^{(t)})>f({\mathtt{x}^*}^{(t)})] \ge p,
\end{eqnarray}
where $p= \frac{1}{4\sqrt{\pi}}$.

\emph{Proof.} 
Applying the anti-concentration of a normal distribution 
\begin{eqnarray}\nn
\Pr[\tilde{f}_{t,b}({\mathtt{x}^*}^{(t)})>f({\mathtt{x}^*}^{(t)})] &=& \Pr\left[ \frac{\tilde{f}_{t,b}({\mathtt{x}^*}^{(t)})-\tilde{\mu}_{t-1}({\mathtt{x}^*}^{(t)})}{\alpha_{t}\tilde{\sigma}_{t-1}({\mathtt{x}^*}^{(t)})}> \frac{f({\mathtt{x}^*}^{(t)})-\tilde{\mu}_{t-1}({\mathtt{x}^*}^{(t)})}{\alpha_{t}\tilde{\sigma}_{t-1}({\mathtt{x}^*}^{(t)})}\right]\\\nn
&\ge& p.
\end{eqnarray}
As a result of the observation that the right hand side of the inequality inside the probability argument is upper bounded by $1$:
\begin{eqnarray}\nn
\frac{f({\mathtt{x}^*}^{(t)})-\tilde{\mu}_{t-1}({\mathtt{x}^*}^{(t)})}{\alpha_{t}\tilde{\sigma}_{t-1}({\mathtt{x}^*}^{(t)})} &\le& \frac{  \frac{1}{2}\alpha_t\tilde{\sigma}_{t-1}({\mathtt{x}^*}^{(t)})+ \frac{1}{2}\alpha_t\epsilon_t }{\alpha_t\tilde{\sigma}_{t-1}({\mathtt{x}^*}^{(t)})}~~~~\text{By}~\Ec_t\\\nn
&\le& 1.~~~\hspace{13em}~\text{By Condition II}~~~~~\square
\end{eqnarray}

\paragraph{Sufficiently Explored Points.} Let $\Sc_t$ denote the set of sufficiently explored points which are unlikely to be selected by S-GP-TS if $\tilde{f}_{t,b}({\mathtt{x}^*}^{(t)})$ is higher than $f({\mathtt{x}^*}^{(t)})$. Specifically, we use the notation
\begin{eqnarray}
\Sc_t = \{x\in D_t:  f(x)+\beta_t\tilde{\sigma}_{t-1}(x)+ \frac{1}{2}\alpha_t\epsilon_t \le f({\mathtt{x}^*}^{(t)})\}. 
\end{eqnarray}  
Recall $\beta_t = \alpha_t(b_t+\frac{1}{2})$. In addition, we define
\begin{eqnarray}
\bar{x}_t ={\text{argmin}}_{x\in D_t\setminus \Sc_t}\tilde{\sigma}_{t-1}(x).
\end{eqnarray}

We showed in equation~\eqref{eq:boundstarsig} that the instantaneous regret can be upper bounded by the sum of standard deviations at $x_{t,b}$ and ${\mathtt{x}^*}^{(t)}$. The standard method based on information gain can be used to bound the cumulative standard deviations at $x_{t,b}$. 
This is not sufficient however because the cumulative standard deviations at ${\mathtt{x}^*}^{(t)}$ do not converge unless there is sufficient exploration around $x^*$. To address this, we use $\bar{x}_t$ as an intermediary to be able to upper bound the instantaneous regret by a factor of $\tilde{\sigma}_{t-1}(x_{t,b})$ through the following lemma. 
\begin{lemma}\label{Lemma:xbar}
Under Condition II, for $t\ge\sqrt{\frac{2}{p}}$, if $\Ec_t$ holds true
\begin{eqnarray}
\tilde{\sigma}_{t-1}(\bar{x}_t)\le \frac{2}{p} \E[\tilde{\sigma}_{t-1}(x_{t,b})],
\end{eqnarray}
where the expectation is taken with respect to the randomness in the sample $\tilde{f}_{t,b}$. 
\end{lemma}
\begin{proof}[Proof of Lemma~\ref{Lemma:xbar}]
First notice that when both $\Ec_t$ and $\tilde{\Ec_t}$ hold true, for all $x\in \Sc_t$
\begin{eqnarray}\nn
\tilde{f}_{t,b}(x)&\le& f(x) + \beta_t \tilde{\sigma}_{t-1}(x)+(\alpha_t-1)\epsilon_t~~~\hspace{2em} \text{by~\eqref{eq:beta}}\\\label{eq:ftildes1}
&\le& f({\mathtt{x}^*}^{(t)}), ~~~~\hspace{5em}~\text{by definition of}~\Sc_t.
\end{eqnarray}
Also, if $\tilde{f}_{t,b}({\mathtt{x}^*}^{(t)})>\tilde{f}_{t,b}(x), \forall x \in S_t$, the rule of selection in TS ($x_{t,b} = \argmax_{x\in \X}\tilde{f}_{t,b}(x)$) ensures $x_{t,b}\in D_t\setminus \Sc_t$. So we have
\begin{eqnarray}\nn
\Pr[x_{t,b} \in D_t\setminus \Sc_t] &\ge& \Pr[\tilde{f}_{t,b}({\mathtt{x}^*}^{(t)})>\tilde{f}_{t,b}(x), \forall x \in S_t]\\\nn
&\ge& \Pr[\tilde{f}_{t,b}({\mathtt{x}^*}^{(t)})>\tilde{f}_{t,b}(x), \forall x \in S_t, \tilde{\Ec_t}] -
\Pr[\bar{\tilde{\Ec}}_t]
\\\nn
&\ge& \Pr[\tilde{f}_{t,b}({\mathtt{x}^*}^{(t)})>f({\mathtt{x}^*}^{(t)})] -
\Pr[\bar{\tilde{\Ec}}_t]~~~\hspace{2em}~\text{by \eqref{eq:ftildes1}}~\\\nn
&\ge& p-\frac{1}{t^2} ~~~\hspace{2em}~\text{by \eqref{eq:plow}}
\\\nn
&\ge&
\frac{p}{2}, ~~~\hspace{2em}~\text{for}~t\ge \sqrt{2/p}.
\end{eqnarray}
Finally, we have
\begin{eqnarray}
\E[\tilde{\sigma}_{t-1}(x_{t,b})] &\ge& \E\left[\tilde{\sigma}_{t-1}(x_{t,b})\bigg|x_{t,b}\in D_t\setminus S_t\right] \Pr[x_{t,b}\in D_t\setminus S_t]\\\nn
&\ge& \frac{p\tilde{\sigma}_{t-1}(\bar{x}_t)}{2},
\end{eqnarray}
where the expectation is taken with respect to the randomness in the sample $\tilde{f}_{t,b}$ at time $t$. 
\end{proof}

Now we are ready to bound the simple regret under Condition II using $\bar{x}_t$ as an intermediary. 
Under Condition II, when both $\Ec_t$ and $\tilde{\Ec}_t$ hold true, 
\begin{eqnarray}\nn
f({\mathtt{x}^*}^{(t)}) - f(x_{t,b}) &=& f({\mathtt{x}^*}^{(t)}) - f(\bar{x}_t)+f(\bar{x}_t) - f(x_{t,b}) \\\nn
&\le& \beta_t\tilde{\sigma}_{t-1}(\bar{x}_t)+ \frac{1}{2}\alpha_t\epsilon_t
+f(\bar{x}_t) - f(x_{t,b})~~~\text{by definition of}~\Sc_t\\\nn
&\le & \beta_t\tilde{\sigma}_{t-1}(\bar{x}_t)+ \frac{1}{2}\alpha_t\epsilon_t\\\nn
&&~~~~~+ \tilde{f}_{t,b}(\bar{x}_t) + \beta_t\tilde{\sigma}_{t-1}(\bar{x}_t) - \tilde{f}_{t,b}(x_{t,b}) +\beta_t\tilde{\sigma}_{t-1}(x_{t,b})+\alpha_t\epsilon_t~~~\text{by~\eqref{eq:beta}}\\\nn
&\le& \beta_t(2\tilde{\sigma}_{t-1}(\bar{x}_t)+\tilde{\sigma}_{t-1}(x_{t,b}) ) + \frac{3}{2}\alpha_t\epsilon_t, ~~~\text{by the rule of selection in TS}.
\end{eqnarray}

Thus, since $f(x^*)-f(x_{t,b})\le 2B$, under Condition II, for $t\ge \sqrt{\frac{2}{p}}$
\begin{eqnarray}
\E[f({\mathtt{x}^*}^{(t)})-f(x_{t,b})]\le \frac{(4+p)\beta_t}{p}\E[\tilde{\sigma}_{t-1}(x_{t,b})] +\frac{3}{2}\alpha_t\epsilon_t + \frac{4B}{t^2}
\end{eqnarray}
where we used Lemma~\ref{Lemma:xbar} and $\Pr[\bar{{\Ec}}_t ~\text{or}~\bar{\tilde{\Ec}}_t  ]\le \frac{2}{t^2}$.

\paragraph{Upper bound on regret.} From the upper bounds on instantaneous regret under Condition I and Condition II we conclude that, for $t\ge \sqrt{\frac{2}{p}}$
\begin{eqnarray}
\E[f({\mathtt{x}^*}^{(t)})-f(x_{t,b})]&\le& \max \bigg\{ \beta_t\tilde{\sigma}_{t-1}(x_{t,b}) + (\beta_t+\alpha_t)\epsilon_t +\frac{4B}{t^2}, \\\nn &&~~~~~\frac{(4+p)\beta_t}{p}\E[\tilde{\sigma}_{t-1}(x_{t,b})] +\frac{3}{2}\alpha_t\epsilon_t + \frac{4B}{t^2}  \bigg\}\\\nn
&\le& \frac{(4+p)\beta_t}{p}\E[\tilde{\sigma}_{t-1}(x_{t,b})] + (\beta_t+\alpha_t)\epsilon_t+ \frac{4B}{t^2}.
\end{eqnarray}
We can now upper bound the cumulative regret. Noticing $\lceil\sqrt{\frac{2}{p}}\rceil = 4$. 
\begin{eqnarray}\nn
R(T,B;f) &=& \sum_{t=1}^T\sum_{b=1}^B \E[f(x^*)-f(x_{t,b})]\\\nn
&=& \sum_{t=1}^{4}\sum_{b=1}^B \E[f(x^*)-f(x_{t,b})]+ \sum_{t=5}^T\sum_{b=1}^B\E[f(x^*)-f(x_{t,b})]\\\nn
&\le&8B\Bc +  \sum_{t=5}^T(\E[f({\mathtt{x}^*}^{(t)})-f(x_{t,b})]+\frac{1}{t^2})\\\nn
&\le&8B\Bc +  \sum_{t=5}^T\sum_{b=1}^B \left(\frac{(4+p)\beta_t}{p}\E[\tilde{\sigma}_{t-1}(x_{t,b})] + (\beta_t+\alpha_t)\epsilon_t+ \frac{4\Bc+1}{t^2}\right)\\\nn
&\le& 8B\Bc + \frac{\pi^2B(4\Bc+1)}{6} + \frac{(4+p)\beta_T}{p}\sum_{t=1}^T\sum_{b=1}^B\E[\tilde{\sigma}_{t-1}(x_{t,b})] + (\beta_T+\alpha_T)\sum_{t=1}^T\sum_{b=1}^B\epsilon_t\\\nn
&\le& 15B\Bc+2B + 30\beta_T\sum_{t=1}^T\sum_{b=1}^B(\bar{a}\E[\sigma_{t-1}(x_{t,b})]+\epsilon_t)
+ (\beta_T+\alpha_T){\epsilon}TB\\\nn
&\le& 15B\Bc+2B + 30\bar{a}\beta_T\sum_{t=1}^T\sum_{b=1}^B\E[\sigma_{t-1}(x_{t,b})]+ 30\beta_T\epsilon TB
+ (\beta_T+\alpha_T){\epsilon}TB\\\nn
&\le& 15B\Bc+2B
+ 30\bar{a}\beta_T\sum_{t=1}^T\sum_{b=1}^B\E[\sigma_{t-1}(x_{t,b})]+
(31\beta_T+\alpha_T)\epsilon TB.
\end{eqnarray}

We simplified the expressions by $\frac{4+p}{p}\le 30$, $\frac{4\pi^2}{6}\le 7$ and $\frac{\pi^2}{6}\le 2$.

We now use a technique based on information gain to upper bound $\sum_{t=1}^T\sum_{b=1}^B\E[\sigma_{t-1}(x_{t,b})]$ as formalized in the following lemma.

\begin{lemma}\label{cumbatch}
For all batch observation sequences $\{x_{t,b}\}_{t\in[T],b\in[B]}$, we have
\begin{eqnarray}
\sum_{t=1}^T\sum_{b=1}^B \sigma_{t-1}(x_{t,b})\le B\sqrt{\frac{2 T\gamma_T}{\log (1+\frac{1}{\tau})}}
\end{eqnarray}
\end{lemma}
\begin{proof}[Proof of Lemma~\ref{cumbatch}]

Without loss of generality assume that at each time instance $t=1,2,\dots, T$, the batch observations are ordered such that $\sigma_{t-1}(x_{t,1})\ge\sigma_{t-1}(x_{t,b}) $, for all $b\in [B]$. We thus have 
\begin{eqnarray}\label{ksn1}
\sum_{t=1}^T\sum_{b=1}^B \sigma_{t-1}(x_{t,b}) \le B\sum_{t=1}^T\sigma_{t-1}(x_{t,1}). 
\end{eqnarray}

For the sequence of observations $\{x_{t,1}\}_{t=1}^T$, define the conditional posterior mean and variance
\begin{eqnarray}\nn
\bar{\mu}_{t}(x) &=& \E[\hat{f}(x)|\{x_{s,1}\}_{s=1}^t]\\\nn
\bar{\sigma}^2_{t}(x) &=& \E[(\hat{f}(x) -\bar{\mu}_{t}(x) )^2| \{x_{s,1}\}_{s=1}^t].
\end{eqnarray}

By the expression of posterior variance of multivariate Gaussian random variables and by positive definiteness of the covariance matrix, we know that conditioning on a larger set reduces the posterior variance. Thus $\bar{\sigma}_{t}(x)\ge \sigma_{t}(x)$. Notice that $\sigma_{t}(x)$ is the posterior variance conditioned on full batches of the observations while $\bar{\sigma}_{t}(x)$ is the posterior variance conditioned on only the first observation at each batch. We thus have
\begin{eqnarray}\label{ksn2}
\sum_{t=1}^T\sigma_{t-1}(x_{t,1}) \le \sum_{t=1}^T\bar{\sigma}_{t-1}(x_{t,1})
\end{eqnarray}

We can now follow the standard steps in bounding the cumulative standard deviation in the non-batch setting. In particular using Cauchy-Schwarz inequality, we have 
\begin{eqnarray}\label{ksn3}
\sum_{t=1}^T\bar{\sigma}_{t-1}(x_{t,1}) \le \sqrt{T\sum_{t=1}^T\bar{\sigma}^2_{t-1}(x_{t,1})}. 
\end{eqnarray}

In addition,~\cite{srinivas2010gaussian} showed that
\begin{eqnarray}\label{ksn4}
\sum_{t=1}^T\bar{\sigma}^2_{t-1}(x_{t,1}) \le \frac{2\gamma_T}{\log(1+\frac{1}{\tau})}. 
\end{eqnarray}

Combining~\eqref{ksn1},~\eqref{ksn2},~\eqref{ksn3} and~\eqref{ksn4}, we arrive at the lemma. 



\end{proof}

We thus have
\begin{eqnarray}
R(T;\text{S-GP-TS}) \le 
 30\bar{a}\beta_TB\sqrt{\frac{2T\gamma_T}{\log(1+\frac{1}{\tau})}}+
(31\beta_T+\alpha_T)\epsilon TB + 15B\Bc+2B
\end{eqnarray}

which can be simplified to
\begin{eqnarray}
R(T;\text{S-GP-TS}) = \tilde{O}\left( \underline{a}\bar{a}(1+c)B\sqrt{T\gamma_T } + \underline{a}^2(1+c^2)\epsilon TB  \right).
\end{eqnarray}

  $\square$

\subsection{Proof of Lemma~\ref{Lemma:ConIneqApprox}}

It remains to prove the concentration inequality for the approximate statistics given in Lemma~\ref{Lemma:ConIneqApprox}.

\begin{proof}[Proof of Lemma~\ref{Lemma:ConIneqApprox}]

By triangle inequality we have
\begin{eqnarray}\nn
|f(x) - \tilde{\mu}_t(x)| &\le& |f(x)-\mu_t(x)|  + |\tilde{\mu}_t(x) - \mu_t(x)| \\\nn
&\le& |f(x)-\mu_t(x)| + c_t\sigma_t(x)~~~~\text{by Assumptions~\ref{Ass2}}.
\end{eqnarray}
From Theorem 2 of~\cite{Chowdhury2017bandit}, 
with probability at least $1-\delta$,
\begin{eqnarray}\nn
f(x)- \mu_t(x)\le\left( B+R\sqrt{2(\gamma_{t}+1+\log(1/\delta))} \right)\sigma_t(x).
\end{eqnarray}

Thus,
\begin{eqnarray}\nn
|f(x)-\tilde{\mu}_t(x)| 
&\le& \left( B+R\sqrt{2(\gamma_{t}+1+\log(1/\delta))} \right) \sigma_t(x) + c_t {\sigma}_{t}(x)
\\\nn
&\le& \underline{a}_t(B+R\sqrt{\frac{2\ln(1/\delta)}{\tau}} +c_t) (\tilde{\sigma}_t(x)+\epsilon_t),
\\\nn
\end{eqnarray}
where the last inequality holds by Assumption~\ref{Ass1}.
\end{proof}


\subsection{Proof of Proposition~\ref{Prop:paramters}}

Here, we use $\tilde{\mu}_t$ and $\tilde{\sigma}_t$ to specifically denote the approximate posterior mean and the approximate posterior standard deviations of the decomposed sampling rules~\eqref{SampIndPoints} and~\eqref{SampIndVariables} in contrast to Sec.~\ref{Sec:RegSec} where we used the notation more generally for any approximate model. We also use $\mu^{(s)}_t$ and $\sigma^{(s)}_t$ to refer to the posterior mean and the posterior standard deviation of SVGP models, and $\mu^{(w)}$ and $\sigma^{(w)}$ to refer to the priors generated from an $M-$truncated feature vector. 
For the approximate posterior mean, we have $\tilde{\mu}_t = \mu^{(s)}_t$. However, the approximate posterior standard deviations $\sigma^{(s)}$ and $\tilde{\sigma}$ are not the same.

By the triangle inequality we have
\begin{eqnarray}
|\tilde{\sigma}_t(x) - \sigma_t(x)| \le |\tilde{\sigma}_t(x) - \sigma^{(s)}_t(x)| + |\sigma^{(s)}_t(x) - \sigma_t(x)|.
\end{eqnarray}

For the first term,
following the exact same lines as in the proof of Proposition 7 in~\cite{wilson2020efficientsampling}, we have
\begin{eqnarray}
|\tilde{\sigma}^2_t(x)-{\sigma^{(s)}_t}^2(x) | \le C_1m_t |\sigma^2(x) - {\sigma^{(w)}}^2(x) |
\end{eqnarray}
where $C_1 =\max_{1\le t\le T} (1+||K^{-1}_{\Zb_{m_t},\Zb_{m_t}}||_{C(\X^2)})$. 
\cite{wilson2020efficientsampling} proceed to upper bound $|\sigma^2(x) - {\sigma^{(w)}}^2(x) |$ by a constant divided by $\sqrt{M}$. We use a tighter bound based on feature representation of the kernel. Specifically from definition of $\delta_M$ we have that 
\begin{eqnarray}
|\sigma^2(x) - {\sigma^{(w)}}^2(x) | &\le& \sum_{i=M+1}^\infty \lambda_i\bar{\phi}_i^2\\\nn 
&=& \delta_M,
\end{eqnarray}
which results in the following upper bound
\begin{eqnarray}
|\tilde{\sigma}^2_t(x)-{\sigma^{(s)}_t}^2(x) |  \le C_1m_t \delta_M.
\end{eqnarray}
For the standard deviations we have
\begin{eqnarray}\nn
|\tilde{\sigma}_t(x)-\sigma^{(s)}_t(x) | &=& \sqrt{|\tilde{\sigma}_t(x)-\sigma^{(s)}_t(x) |^2 }\\\nn
&\le&\sqrt{|\tilde{\sigma}_t(x)-\sigma^{(s)}_t(x) ||\tilde{\sigma}_t(x)+\sigma^{(s)}_t(x) | }\\\nn
&=& \sqrt{|\tilde{\sigma}^2_t(x)-{\sigma^{(s)}_t}^2(x) |^2 }\\\label{eq:wilson}
&\le& \sqrt{C_1m_t\delta_M},
\end{eqnarray}
where the first inequality holds because $|\tilde{\sigma}_t(x)-\sigma^{(s)}_t(x) |\le |\tilde{\sigma}_t(x)+\sigma^{(s)}_t(x) |$ for positive $ \tilde{\sigma}_t(x)$ and $\sigma^{(s)}_t(x)$.

We can efficiently bound the error in the SVGP approximation based on the convergence of SVGP methods. Let us first focus on the inducing features. It was shown that (Lemma 2 in~\cite{Burt2019Rates}), for the SVGP with inducing features
\begin{eqnarray}
\text{KL}\left(\text{GP}(\mu_t,\sigma_t), {\text{GP}}(\mu^{(s)}_t,k^{(s)}_t)\right)\le \frac{\theta_t}{\tau}.
\end{eqnarray}
where $\text{GP}(\mu_t,\sigma_t)$ and ${\text{GP}}(\mu^{(s)}_t,k^{(s)}_t)$ are the true and the SVGP approximate posterior distributions at time $t$, and KL denotes the  Kullback-Leibler divergence between them. On the right hand side, $\theta_t$ is the trace of the error in the covariance matrix. Specifically, $\theta_t = \text{Tr}(E_t)$ where $E_t = K_{\Xb_t,\Xb_t} - K_{\Zb_t,\Xb_t}^{\TP}K_{\Zb_t,\Zb_t}K_{\Zb_t,\Xb_t} $. Using the Mercer expansion of
the kernel matrix,~\cite{Burt2019Rates} showed that $[E_t]_{i,i} = \sum_{j=m_t+1}^\infty\lambda_j\phi^2_j(x_i)$. Thus 
\begin{eqnarray}
\theta_t &=&\sum_{i=1}^t\sum_{j=m+1}^\infty\lambda_j\phi^2_j(x_i)\\\nn
&\le& t\sum_{j=m_t+1}^{\infty}\lambda_j\bar{\phi}_j^2\\\nn
&=&t\delta_{m_t}
\end{eqnarray}

Thus, 
\begin{eqnarray}\label{KL123}
\text{KL}\left(\text{GP}(\mu_t,\sigma_t), {\text{GP}}(\mu^{(s)}_t,k^{(s)}_t)\right)\le \kappa_t/2.
\end{eqnarray}
where $\kappa_t = 2tB\delta_m/\tau$ that is determined by the number of current observations. 
In comparison, \cite{Burt2019Rates} proceed by introducing a prior distribution on $x_i$ and bounding $[E_t]_{i,i,}$ differently.

For the case of inducing points drawn from an $\epsilon_0$ close k-DPP distribution, similarly following the exact lines as \cite{Burt2019Rates} except for the upper bound on $[E_t]_{i,i}$, with probability at least $1-\delta$, \eqref{KL123} holds with $\kappa_t = \frac{2tB(m_t+1)\delta_{m_t}}{\delta\tau} + \frac{4tB\epsilon_0}{\delta\tau}$ where $\epsilon_0$ that is determined by the number of current observations.

In addition, if the KL divergence between two Gaussian distributions is bounded by $\kappa_t/2$, we have the following bound on the means and variances of the marginals [Proposition 1 in~\cite{Burt2019Rates}]
\begin{eqnarray}\nn
|\mu^{(s)}_t(x)-\mu_t(x)|&\le& {\sigma}_t(x)\sqrt{\kappa_t},\\\label{eq:burt}
|1-\frac{{\sigma^{(s)}_t}^2(x)}{\sigma^2_t(x)}|&\le&\sqrt{3\kappa_t},
\end{eqnarray}
which by algebraic manipulation gives 
\begin{eqnarray}
\sqrt{1-\sqrt{3\kappa_t}}\sigma_t({x}) \le \sigma^{(s)}_t(x) \le \sqrt{1+\sqrt{3\kappa_t}}\sigma_t({x})
\end{eqnarray}
Combining the bounds on $\sigma^{(s)}_t$ with~\eqref{eq:wilson}, we get
\begin{eqnarray}\nn
\sqrt{1-\sqrt{3\kappa_t}}\sigma_t({x})-\sqrt{C_1m_t\delta_M} \le \tilde{\sigma}_t(x) \le \sqrt{1+\sqrt{3\kappa_t}}\sigma_t({x})+\sqrt{C_1m_t\delta_M}
\end{eqnarray}
Comparing this bound with Assumption~\ref{Ass1}, we have $\underline{a}_t  =\frac{1}{\sqrt{1-\sqrt{3\kappa_t}}}$, 
$\bar{a}_t = \sqrt{1+\sqrt{3\kappa_t}}$, and $\epsilon_t = \sqrt{C_1m_t\delta_M}$.
Also, since $\mu^{(s)}_t = \tilde{\mu}_t$, comparing~\eqref{eq:burt} with Assumption~\ref{Ass2}, we have $c_t = \sqrt{\kappa_t}$. 

$\square$



\subsection{Proof of Theorem~\ref{The:MatSE}}

In Theorem~\ref{The:RegretBoundsApproximate}, we proved that \begin{eqnarray}\nn
R(T,B;f) = {O}\left( \underline{a}\bar{a}BR\sqrt{d\gamma_T(\gamma_{TB}+\log(T)) T\log(T)} + \underline{a}\epsilon TBR\sqrt{d(\gamma_{TB}+\log(T))\log(T)}  \right)
\end{eqnarray}
We thus need to show that $\underline{a}\bar{a}$ is a constant independent of $T$ and $\underline{a}\epsilon$ is small so that the second term is dominated by the first term. 

In the case of Mat{\'e}rn kernel, $\lambda_j = O(j^{-\frac{2\nu+d}{d}})$ implies that $\delta_m = O(m^{-\frac{2\nu}{d}})$. Under sampling rule~\eqref{SampIndPoints}, we select $\delta=\frac{1}{T}$ and $\epsilon_0=\frac{1}{T^2\log(T)}$ in Proposition~\ref{Prop:paramters}. We thus need $\kappa_T = O({T^2m_T\delta_{m_T}})$ and $\epsilon_T\sqrt{T} = O(\sqrt{m_T\delta_M T}) $ be sufficiently small constants. That is achieved by selecting $m_T = T^{\frac{2d}{2\nu-d}}$ and $M=T^{\frac{(2\nu+d)d}{2(2\nu-d)\nu}}$.

Under sampling rule~\eqref{SampIndVariables}, we need $\kappa_T = O({T\delta_{m_T}})$ and $\epsilon_T\sqrt{T} = O(\sqrt{m_T\delta_M T}) $ be sufficiently small constants. That is achieved by selecting $m_T = T^{\frac{d}{2\nu}}$ and $M=\frac{(2\nu+d)d}{4\nu^2 }$.

In the case of SE kernel, $\lambda_j = O(\exp(-j^{\frac{1}{d}}))$ implies that $\delta_m = O(\exp(-m^{\frac{1}{d}}))$. Under sampling rule~\eqref{SampIndPoints}, we select $\delta=\frac{1}{T}$ and $\epsilon_0=\frac{1}{T^2\log(T)}$ in Proposition~\ref{Prop:paramters}. We thus need $\kappa_T = O({T^2m_T\delta_{m_T}})$ and $\epsilon_T\sqrt{T} = O(\sqrt{m_T\delta_M T}) $ be sufficiently small constants. That is achieved by selecting $m_T = (\log(T))^{d}$ and $M=(\log(T))^{d}$.
We obtain the same results under 
sampling rule~\eqref{SampIndVariables} where we need $\kappa_T = O({T\delta_{m_T}})$ and $\epsilon_T\sqrt{T} = O(\sqrt{m_T\delta_M T}) $ be sufficiently small constants. 

$\square$

\section{Additional Experiments and Experimental Details}\label{app:Experiments}

In Section \ref{Sec:Exp}, we tested S-GP-TS across popular synthetic benchmarks from the BO literature. We considered the Shekel, Hartmann and Ackley (see Figure \ref{fig:synthetic_extra}) functions, each contaminated by Gaussian noise with variance $0.1,0.5$ and $0.5$, respectively. Note that for Hartmann and Ackley, we chose our observation noise to be an order of magnitude larger than usually considered for these problems in order to demonstrate the suitability of S-GP-TS for controlling large optimization budgets (as required to optimize these highly noisy functions). We now provide explicit forms for these synthetic functions and list additional experimental details left out from the main paper.

\textbf{Shekel function}. A four-dimensional function with ten local and one global minima defined on $\mathcal{X}\in[0,10]^4$:
\begin{align*}
    f(\textbf{x}) = - \sum_{i=1}^{10}\left(\sum_{j=1}^{4}(x_j-A_{j,i})^2+\beta_i\right)^{-1},
\end{align*}
where 
\begin{align*}
    \beta=\begin{pmatrix}
1\\2\\2\\4\\4\\6\\3\\7\\5\\5
\end{pmatrix}
\quad \textrm{and} \quad
    A=\begin{pmatrix}
4 & 1 & 8 & 6 & 3 & 2 & 5 & 8 & 6 & 7  \\
4 & 1 & 8 & 6 & 7 & 9 & 3 & 1 & 2 & 3.6   \\
4 & 1 & 8 & 6 & 3 & 2 & 5 & 8 & 6 & 7  \\
4 & 1 & 8 & 6 & 7 & 9 & 3 & 1 & 2 & 3.6 
\end{pmatrix}.
\end{align*}

\textbf{Ackley function}. A five-dimensional function with many local minima surrounding a single global minima defined on $\mathcal{X}\in[-2,1]^5$:
\begin{align*}
    f(\textbf{x}) = -20 \exp\left(-0.2*\sqrt{\frac{1}{4}\sum_{i=1}^dx_i^2}\right) - \exp\left(\frac{1}{4}\sum_{i=1}^4\cos(2\pi x_i)\right) + 20 +\exp(1).
\end{align*}

\textbf{Hartmann 6 function}. A six-dimensional function with six local minima and a single global minima defined on $\mathcal{X}\in[0,1]^6$:

\begin{align*}
    f(\textbf{x})= -\sum\limits_{i=1}^4\alpha_{i} \exp\left(-\sum\limits_{j=1}^6A_{i,j}(x_j-P_{i,j})^2\right),
\end{align*}

where 
\begin{align*}
    &A=\begin{pmatrix}
10 & 3 & 17 & 3.5 & 1.7 & 8 \\
0.05 & 10 & 17 & 0.1 & 8 & 14 \\
3 & 3.5 & 1.7 & 10 & 17 & 8\\
17 & 8 & 0.05 & 10 & 0.1 & 14
\end{pmatrix},
\quad
&\alpha=\begin{pmatrix}
1 \\
1.2 \\
3 \\
3.2
\end{pmatrix},
\\
\\
    &P=10^{-4}\begin{pmatrix}
1312 & 1696 & 5569 & 124& 8283& 5886\\
2329 & 4135 & 8307 & 3736& 1004& 9991 \\
2348 & 1451 & 3522 & 2883& 3047& 6650 \\
4047 & 8828 & 8732 & 5743& 1091& 381
\end{pmatrix}.
\end{align*}

For all our synthetic experiments (both for S-GP-TS and the baseline BO methods), we follow the implementation advice of \cite{riche2021revisiting} regarding constraining length-scales (to stabilize model fitting) and by maximizing acquisition functions (and Thompson samples) using L-BFGS \citep{liu1989limited} starting from the best location found across a random sample of $500*d$ locations (where $d$ is the problem dimension). Our SVGP models are fit with an ADAM  optimizer \cite{kingma2014adam} with an initial learning rate of 0.1, ran for at most 10,000 iterations but with an early stopping criteria (if 100 successive steps lead to a loss less that 0.1). We also implemented a  learning rate reduction factor of 0.5 with a patience of 10. Our implementation of the GIBBON acquisition function follows \cite{moss2021gibbon} and is built on 10 Gumbel samples built across a grid of 10,000 *$d$ query points. For BO's initialization step, our S-GP-TS models are given a single random sample of the same size as the considered batches and standard BO routines are given $d+4$ initial samples (again following the advice of \cite{riche2021revisiting}). The function evaluations required for these initialization are included in our Figures.

\subsection{S-GP-TS on the Ackley Function}

To supplement the synthetic examples included in the main body of the paper, we now consider the performance of S-GP-TS when used to optimize the challenging Ackley function, defined over 5 dimensions and under very high levels of observation noise (Gaussian with variance $0.5$). The Ackley function (in 5D) has thousands of local minima and a single global optima in the centre. As this global optima has a very small volume, achieving high precision optimization on this benchmark requires high levels of exploration (akin to an active learning task). Figure \ref{fig:synthetic_extra} demonstrates the performance of S-GP-TS on the Ackley benchmark, where we see that S-GP-TS is once again able to find solutions with lower regret than the sequential benchmarks and effectively allocate batch resources. In contrast to our other experiments, where the K-means inducing point selection routine significantly outperforms greedy variance reduction, our Ackley experiment shows little difference between the different inducing point selection routines. In fact, greedy variance selection  slightly outperforms selection by k-means. We hypothesize that the strong repulsion properties of DPPs (as approximated by greedy variance selection) are advantageous for optimization problems requiring high levels of exploration.

\begin{figure}[!ht]
 \centering
   \includegraphics[trim=0mm 3mm 0mm 2mm, clip, height=50mm]{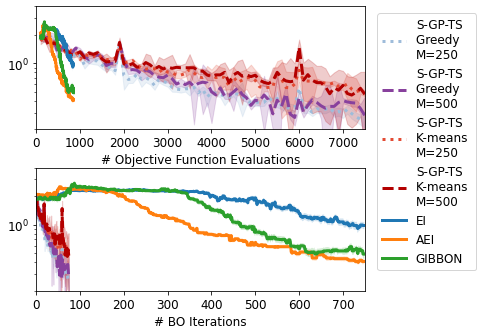}
   \caption{Simple regret on  5D Ackley function. The best S-GP-TS approaches are able to efficiently allocate additional optimization budgets to achieve lower final regret than the sequential baselines. When considering regret with respect to the BO iteration (bottom panels,idealised parallel setting), S-GP-TS achieves low regret in a fraction of the iterations required by standard BO routines. For this task, the choice of inducing point selection strategy (and number of inducing points) is not as crucial as for our other synthetic benchmarks, however, greedy variance selection provides a small improvement over selection by k-means.
   }\label{fig:synthetic_extra}
   \vspace{-2mm}
\end{figure}

\subsection{A Comparison of S-GP-TS with other batch BO routines}

To accompany Figures \ref{fig:synthetic} and \ref{fig:synthetic_extra} (our comparison of S-GP-TS with sequential BO routines), we also now compare S-GP-TS with popular batch BO routines. Once again, we stress that these existing BO routines do not scale to the large batch sizes that we consider for S-GP-TS, and so we plot their performance for  $B=25$ (a batch size considered  large in the context of these exiting BO methods). We consider two well-known batch extensions of EI: Locally Penalized EI \citep[LP,][]{gonzalez2016batch} and the multi-point EI (known as qEI) of \cite{chevalier2013fast}. We also consider with a recently proposed batch information-theoretic approach known as General-purpose Information-Based Bayesian OptimizatioN \cite[GIBBON,][]{moss2021gibbon}. The large optimization budgets considered in these problems prevent our use of batch extensions of other popular but high-cost acquisition functions such as those based on knowledge gradients \citep{wu2016parallel} or entropy search \citep{hernandez2017parallel}. Figure \ref{fig:synthetic_batch} compares our S-GP-TS methods (B=100) with the popular batch routines (B=25), where we see that S-GP-TS achieves lower regret than existing batch BO methods for our most noisy synthetic function (Hartmann).

\begin{figure}[!ht]
 \centering
   \includegraphics[trim=0mm 3mm 0mm 2mm, clip, height=50mm]{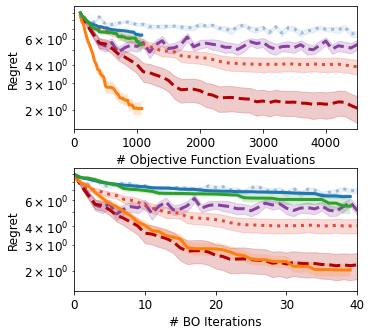}
   \includegraphics[trim=0mm 3mm 0mm 2mm, clip, height=50mm]{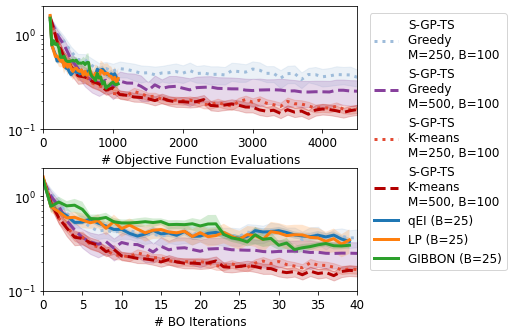}
   \caption{ Simple regret on Shekel (4D, left) and Hartmann (6D, right) as a function of either the number of evaluations (top) or BO iterations (bottom). S-GP-TS methods are ran for batches of size $B=100$ and the batch BO methods for batches of size $B=25$. We see that S-GP-TS is particularly effective when performing the batch optimization of particularly noisy functions (Hartmann), exceeding the regret of the batch baselines. In our synthetic benchmark with low observation noise (Shekel), S-GP-TS is less efficient in terms of individual function evaluations, however, S-GP-TS 's ability to control larger batches means that it can match the performance of the highly perfomant LP with respect to the number of BO iterations (the idealised parallel setting).
   }\label{fig:synthetic_batch}
   \vspace{-2mm}
\end{figure}

\end{document}


%

%


\section{Introduction}\label{Intro}

Over the past decade, Thompson sampling \citep[TS,][]{Thompson1933} has become a very popular algorithm to solve sequential optimization (here, maximization) problems. TS proceeds by sequentially sampling from the posterior distribution of the optimum of the objective function,
allowing for efficiently addressing exploration-exploitation trade-offs.
In the bandit setting, TS has been shown to enjoy favorable performance both in theory and practice
(see~\cite{Kaufmann2012,Agrawal2013TS} for Bernoulli distributions with beta priors,~\cite{Agrawal2013TS} for Gaussian distributions with Gaussian priors,~\cite{Kaufmann2013TS} for the one-dimensional exponential family with uninformative priors, \cite{Gopalan2014TS} for finitely supported distributions and priors,~\citep{Abeille2017LinTS}  for linear bandits, and~\cite{Chowdhury2017bandit} for kernelized bandits).

Gaussian Processes (GPs) provide a spectrum of powerful and flexible modeling tools that can be used to provide predictions of the objective function,
leading to very efficient optimization algorithms in a small data scenario \citep{Shahriari2016}.
TS using GP models \citep[GP-TS,][]{Chowdhury2017bandit} has been found to be empirically competitive while providing guarantees similar to the best GP-based algorithms. In addition, a particular advantage of GP-TS is its versatility (see \citet{hernandez2017parallel,kandasamy2018parallelised}  and \citet{paria2019flexible} for batch-sequential and multi-objective versions, respectively).

Despite those appealing features, the practical implementation of GP-TS introduces two 
main difficulties that prevent the method from scaling in terms of time horizon. 
First, building the GP posterior distribution, which is needed every time a new observation is acquired, is well known to require an $O(t^3)$ computation (where $t$ is the number of observations) due to a matrix inversion step~\citep{Rasmussen2006}.
Given the updated posterior model, sampling from its optimum is the second challenging task. The standard approach is to draw a joint sample on a grid discretizing the search space. Sampling over a grid of size $N$ has an $O(N^3)$ complexity \citep[due to a Cholesky decomposition step,][]{diggle1998model}, which is the second computational bottleneck.

A natural answer to the first challenge is to rely on the recent advances in sparse variational GP models (referred to as SVGP for the rest of the paper) 
which allow a low rank approximation of the GP posterior in $O(m^2 t)$ computation, where $m$ is the number of the so-called \emph{inducing variables} and grows at a rate much slower than $t$~\citep{Titsias2009Variational}. 
The inducing variables are manifested either as \emph{inducing points} or \emph{inducing features} \citep[sometimes referred to as inducing inter-domain variables,][]{Burt2019Rates,vanderWilk2020Framework,dutordoir2020}.
Furthermore, \cite{wilson2020efficientsampling} introduced an efficient sampling rule (referred to as \emph{decoupled } sampling) which decomposes a sample from the posterior into the sum of a prior with $M$ features (see Sec.~\ref{Subsec:feature}) and an SVGP update, reducing the computational cost of drawing a sample to $O\left((m+M)N\right)$. Leveraging this sampling rule
results in a scalable GP-TS algorithm (henceforth S-GP-TS) that is simple to implement and can handle orders of magnitude larger number of observation points. 

The main question addressed here is to assess whether such an approach maintains the performance guarantees (in terms of regret, see~\eqref{eq:regdef}) of the vanilla GP-TS.
Indeed, using sparse models and decoupled sampling introduce two layers of approximation,
that, if handled without care, can have a dramatic effect on performance \citep[see e.g.][that showed that even a small constant posterior error in $\alpha-$divergence can lead to under- or over-exploration and poor performance, i.e. linear regret]{Phan2019TSExample}.
Hence, this work proposes the analysis of S-GP-TS methods, along with necessary conditions on some algorithmic parameters to guarantee performance. 


\subsection{Contributions}
In Theorem~\ref{The:RegretBoundsApproximate}, we establish an upper bound on the regret of TS using any approximate GP model that satisfy some conditions on the quality of their posterior approximations (Assumptions~\ref{Ass1} and~\ref{Ass2}). This bound is independent of the particular sampling rule and may be of general interest.

Then, focusing on SVGP models, we provide bounds for the number $M$ of prior features and the number $m$ of inducing variables required to guarantee a low regret when the decoupled sampling rule is used. The bounds on $m$ and $M$ are characterized based on the spectrum of the GP kernel. With these bounds in place, we report drastic improvement in the computational complexity of S-GP-TS with Mat{\'e}rn and Squared Exponential kernels (see Table~\ref{Table:CC}) without compromising the regret performance. Specifically, we prove $O(\gamma_T\sqrt{T\log(T)})$ regret bounds (where $\gamma_T$ is the so-called \textit{information gain}, see Sec.~\ref{Sec:Analysis}), which is the same order as the regret for vanilla GP-TS~\citep{Chowdhury2017bandit}, when $m$ and $M$ are large enough so that conditions on the quality of the approximation are satisfied.

The details of the problem formulation and the analysis are provided in Sec.~\ref{Sec:PF} and Sec.~\ref{Sec:Analysis}, respectively. The preliminaries on GPs are reviewed in Sec.~\ref{Sec:GPs} and the details of S-GP-TS methods are presented in Sec.~\ref{Sec:Alg}.

\subsection{Related work}
Our sampling rule is based on the decoupled sampling rule of~\cite{wilson2020efficientsampling}. Our contribution is to provide analysis for TS using a version of decoupled sampling. 
Note however that our sampling rule differs slightly from the one of \citet{wilson2020efficientsampling}, as we scale the covariance of the sample to ensure sufficient exploration required to guarantee the convergence in TS algorithms. 
The proof of Theorem~\ref{The:RegretBoundsApproximate} is inspired by the analysis of vanilla GP-TS given in~\cite{Chowdhury2017bandit}. 
Our analysis however tackles additional challenges due to the introduced approximations (We characterize the regret bounds in terms of a parameterization of the approximation). An additive approximation error in the standard deviation, in particular, needs special attention as it might cause under-exploration in the vicinity of the optimum point. 
To bound the values of $m$ and $M$ guaranteeing a low regret, we use a result from~\cite{Burt2019Rates} on the KL-divergence between the SVGP model and the exact GP model.

\cite{Calandriello2019Adaptive} introduced a scalable version of GP-UCB (an \emph{optimistic} optimization algorithm that selects the observation points according to an upper confidence bound score), based on randomized matrix sketching
and leverage score sampling.
They proved the same regret guarantees as the ones for GP-UCB (up to a multiplicative $\log(T)$ factor). In comparison, the analysis of S-GP-TS is different and has additional complexity due to the cost of posterior sampling ($O(N^3)$ per step). 

Among other approaches to sparse approximation of GP posteriors is to use a finite dimensional truncated feature representation~\citep{Hernandez2014features, Shahriari2016} such as random Fourier features (FFs) which might suffer from \emph{variance starvation}, i.e., underestimate the variance of points far from the observations~\citep{wang2018batched, Mutny2018SGPTS, wilson2020efficientsampling}. Intuitively, that is because the Fourier basis is only an efficient basis for representing stationary GPs, while the posterior is generally nonstationary~\citep{wilson2020efficientsampling}. To the best of our knowledge, no regret guarantees are reported for this approach. 

\cite{Mutny2018SGPTS} however took a different approach based on the quadratic FFs (in contrast to the random FFs) and constructed scalable BO methods not suffering from variance starvation. They provided regret bounds for the case of additive SE kernel which match the ones resulting from the application of our Theorem~\ref{The:RegretBoundsApproximate} to the SE kernel.
The application of our regret bounds to $d-$dimensional linear kernels results in the same $O(d\sqrt{T\log(T)})$ bounds as in linear bandits with TS~\citep{Abeille2017LinTS} or with UCB~\citep{Abbasi2011}. 

\cite{hernandez2017parallel} introduced a heuristic parallel implementation of TS where each updated posterior model is used to draw $n'$ samples. \cite{kandasamy2018parallelised} showed that parallelization improves the computational complexity due to updating the model, by a constant factor of $O(n')$ at the price of incurring an $O(n')$ additive regret. The parallelization however does not reduce the cost due to sampling. This idea directly applies to our results as well.



\section{Problem Formulation}\label{Sec:PF}

We consider the sequential optimization of an unknown function $f$ over a compact set $\X\subset \Rr^d$. A sequential learning policy selects a batch of $B$ observation points $\{x_{t,b}\}_{b\in [B]}$ at each time step $t=1,2,\dots,T$ and receives the corresponding real-valued and noisy rewards $\{y_{t,b}=f(x_{t,b})+\epsilon_{t,b}\}_{b\in[B]}$, where $\epsilon_{t,b}$ denotes the observation noise. Throughout the paper, we use the notation $[n]=\{1,2,\dots,n\},$ for $n\in \Nn$. 
As is common in both the bandits and GP literature, our analysis uses the following sub-Gaussianity assumption,  a direct consequence of which is that $\E[\epsilon_{t,b}] = 0$, for all~$t,b\in \Nn$.
\begin{assumption}\label{AssNoise}
$\epsilon_{t,b}$ are i.i.d., over both $t$ and $b$, $R-$sub-Gaussian random variables, where $R>0$ is a fixed constant. Specifically,
$
\E[e^{h\epsilon_{t,b}}]\le \exp(\frac{h^2R^2}{2}),~\forall h\in \Rr, \forall t,b\in \Nn. 
$

\end{assumption}

Let $x^*\in\argmax_{x\in\X}f(x)$ be an optimal point.
We can then measure the performance of a sequential optimizer by its \emph{strict regret}, defined as the cumulative loss compared to $f(x^*)$ over a time horizon $T$
\begin{eqnarray} \label{eq:regdef}
R(T,B;f) = \E\left[\sum_{t=1}^T \sum_{b=1}^B f(x^*) - f(x_{t,b})\right],
\end{eqnarray}
where the expectation is with respect to the randomness in noise and the possible stochasticity in the sequence of the selected batch observation points $\{x_{t,b}\}_{t\in[T], b\in [B]}$. 
Note that our regret measure (\ref{eq:regdef}) is defined for the true unknown  $f$. In contrast, the alternative {Bayesian regret} \citep[see e.g.][]{Russo2018TutorialTS, kandasamy2018parallelised} averages over a prior distribution for $f$. As upper bounds on strict regret directly apply to the Bayesian regret (but not necessarily the reverse), our results are stronger than those that can be achieved when analysing just Bayesian regret, for example when applying the technique of \cite{Russo2014} that equates TS's Bayesian regret with that of the well-studied upper confidence bound policies.

Following \cite{Chowdhury2017bandit,srinivas2010gaussian, Calandriello2019Adaptive}, our analysis assumes a regularity condition on the objective function motivated by kernelized learning models and their associated reproducing kernel Hilbert spaces \citep[RKHS,][]{berlinet2011reproducing}: 
\begin{assumption}\label{AssNorm}
Given an RKHS $H_k$, the norm of the objective function is bounded: $||f||_{H_k}\le \Bc$, for some $\Bc>0$, and $k(x,x')\le 1$, for all $x,x'\in\X$.
\end{assumption}

In the case of practically relevant kernels, Assumption~\ref{AssNorm} implies certain smoothness properties for the objective functions.

\section{Gaussian Processes and Sparse Models}\label{Sec:GPs}

GPs are powerful non-parametric Bayesian models over the space of functions \citep{Rasmussen2006} with a distribution specified by a mean function $\mu(x)$
(henceforth assumed to be zero for simplicity)
and a positive definite kernel (or covariance function) $k(x,x')$.
We provide here a brief description of the classical GP model and two sparse variational formulations.

\subsection{Exact Gaussian Process models}

Suppose that we have collected a set of location-observation tuples $\H_{t} = \{\Xb_{t},\yb_{t}\}$, where $\Xb_{t}$ is the $tB\times d$ matrix of locations with rows $[\Xb_{t}]_{(s-1)B+b} = x_{s,b}$, and $\yb_{t}$ is the $tB$-dimensional column vector of observations with elements $[\yb_{t}]_{(s-1)B+b}=y_{s,b}$, for all $s\in[t], b\in[B]$.
Then, assuming a  Gaussian observation noise
, the posterior of the GP model $\hat{f}$ 
given the set of past observations $\H_{t}$, is also a GP with mean $\mu_{t}(\cdot)$, variance $\sigma^2_{t}()$ and kernel function $k_t(\cdot,\cdot)$
specified as
\begin{eqnarray}
\mu_t(x) =  k^{\TP}_{\Xb_t,x} (K_{\Xb_t,\Xb_t}+\tau \Ib)^{-1} \yb_{t},
\quad
k_{t}(x,x') =  k(x,x') -  k^{\TP}_{\Xb_t,x} (K_{\Xb_t,\Xb_t}+\tau \Ib)^{-1} k_{\Xb_t,x'},\label{GPt}
\end{eqnarray}
and $\sigma^2_{t}(x)= k_{t}(x,x)$, with $k_{\Xb_t,x}$ the $tB$ dimensional column vector with entries $[k_{\Xb_t,x}]_{(s-1)B+b}= k(x_{s,b},x)$, and $K_{\Xb_t,\Xb_t}$ the ${tB}\times{tB}$ positive definite covariance matrix with entries $[K_{\Xb_t,\Xb_t}]_{(s-1)B+b,(s'-1)B+b'} = k(x_{s,b}, x_{s',b'})$. 
We directly see from (\ref{GPt}) that accessing the posterior expressions require an $O((tB)^3)$ matrix inversion, which is a computational bottleneck for large values of $tB$. 

Note that in our problem formulation $f$ is fixed and observation noise has an unknown sub-Gaussian distribution. Using  a GP prior and assuming a Gaussian noise is merely for ease of modelling and does not affect our assumptions on $f$ and $\epsilon_{t,b}$. The notation  $\hat{f}$ is thus used to distinguish the GP model from the fixed $f$.




\subsection{Sparse Variational Gaussian Process Models with Inducing Points}\label{Sec:SVGP}

To overcome the cubic cost of exact GPs, SVGPs \citep{Titsias2009Variational,Hensman2013} instead approximate the GP posterior  
through a  set of \emph{inducing points} $\Zb_t = \{z_1,..., z_{m_t}\}$ ($z_i \in \X$, with $m_t << tB$). Conditioning on the \emph{inducing variables}
$\ub_t = \hat{f}(\Zb_t)$ (rather than the $tB$ observations in $\yb_t$) and specifying a prior Gaussian density $q_{t}(\ub_t) = \Nc(\mb_t, \Sb_t)$, yields an approximate posterior distribution that, crucially, is still  
a GP but with the significantly reduced computational complexity of $O(m_t^2t)$. The posterior mean and covariance of the SVGP are given in closed form as
\begin{eqnarray}\nn
{\mu}^{(s)}_t(x) = k_{\Zb_t,x}^{\TP} K^{-1}_{\Zb_t,\Zb_t}\mb_t\quad
{k}^{(s)}_t(x, x') = k(x,x')+ k_{\Zb_t,x}^{\TP}K^{-1}_{\Zb_t,\Zb_t} (\Sb_t - K_{\Zb_t,\Zb_t})K^{-1}_{\Zb_t,\Zb_t} k_{\Zb_t,x'}.\nn
\end{eqnarray}
The variational parameters $\mb_t$ and $\Sb_t$ are set as the maximizers of the evidence lower bound (ELBO, see Appendix~\ref{app:SVGP} for the details) and can be 
optimized numerically with mini-batching \citep{Hensman2013}. There are various standard ways in practice to select the locations of the inducing points $\Zb_t$, e.g. by using an experimental design, sampling from a k-DPP (that stands for determinantal point process),
or by optimizing them along with the inducing variables.

\subsection{Sparse Variational Gaussian Process Models with Inducing Features}\label{Subsec:feature}

An alternative approximation strategy is using inducing feature approximations \citep{hensman2017variational,Burt2019Rates,Dutordoir2020spherical}. Here, we
define inducing variables as the linear integral transform of $\hat{f}$ with respect to some \emph{inducing features} \citep{lazaro2009inter} $\psi_1(x),..,\psi_{m_t}(x)$, i.e we set our $i^{\textrm{th}}$ inducing variable as  $u_{t,i} = \int_{\X} \hat{f}(x)\psi_i(x)dx$.  Courtesy of Mercer's theorem, we can {usually} decompose our chosen kernel $k$ as the inner product of possibly infinite dimensional feature maps (see Theorem 4.1 in \cite{Kanagawa2018}) to provide the expansion $k(x,x') = \sum_{j=1}^{\infty}\lambda_j\phi_j(x).\phi_j(x')$ for eigenvalues $\{\lambda_j\in \Rr^+\}_{j=1}^{\infty}$ and eigenfunctions $\{\phi_j\in H_k\}_{j=1}^\infty$. If we set our inducing features to be the $m_t$ eigenfunctions with largest eigenvalues, it can be shown that  $\text{cov}(u_{t,i}, u_{t,j}) = \lambda_j \delta_{i,j}$ and $\text{cov}(u_{t,j}, \hat{f}(x)) = \lambda_j \phi_j(x)$ , yielding an approximate Gaussian Process model with  posterior mean and covariance given by
\begin{eqnarray}\nn
{\mu}^{(s)}_t(x) = \phib_{m_t}^{\TP}(x)\mb_t\qquad {k}^{(s)}_t(x,x') = k(x,x') +   \phib_{m_t}^{\TP}(x)(\Sb_t-\Lambda_{m_t})\phib_{m_t}(x').
\end{eqnarray}
Here, $\mb_t$ and $\Sb_t$ are inducing parameters (as above),  $\phib_{m_t}(x)\triangleq [\phi_1(x),...,\phi_{m_t}(x)]^{\TP}$ is the truncated feature vector and $\Lambda_{m_t}$ is the ${m_t}\times {m_t}$ diagonal matrix of eigenvalues, $[\Lambda_{m_t}]_{i,j}=\lambda_i\delta_{i,j}$. 


Inducing feature approximations have strong advantages, in particular a reduced computational cost and the fact that no inducing points need to be specified. However, accessing these eigenfeatures
require the Mercer decomposition of the used kernel, which is available for certain kernels on manifolds \citep{Borovitskiy2020,Dutordoir2020spherical}, but limited to low dimensions for others \citep{zhu1997gaussian,solin2020hilbert}.

\section{Scalable Thompson Sampling using Gaussian Process Models (S-GP-TS)}\label{Sec:Alg}

At each time $t$, an ideal GP-TS proceeds by drawing $B$ i.i.d. samples $\{\hat{f}_{t,b}\}_{b\in [B]}$  from the posterior distribution of $\hat{f}$ (up to a scaling of its covariance) and finding their maximizers. 
Since $\hat{f}_{t,b}$ is an infinite dimensional object, it is standard in practice \citep{kandasamy2018parallelised} to resort to sampling on a discretization $D_t$ of $\X$.
Then, the $x_{t,b}$ are selected as the maximizers of the discretized samples:
\begin{eqnarray}
    \{x_{t,b} = \argmax_{x\in D_t}\hat{f}_{t,b}(x)\}_{b\in [B]} \label{eq:maxsample}
\end{eqnarray}

The discrete draws are generated via an affine transformation of Gaussian random variables $\xi\sim\Nc(\mathbf{0}_{N_t},\Ib_{N_t})$, where for $N_t=|D_t|$, $\mathbf{0}_{N_t}$ is the $N_t\times 1$ zero vector and $\Ib_{N_t}$ is the $N_t\times N_t$ identity matrix. Specifically, 
\begin{equation}
    \hat{f}_{D_t}|\yb_t = \mu_{D_t|\yb_t} + \alpha_tK^{\frac{1}{2}}_{D_t,D_t|\yb_t}\xi,
\end{equation}
where $\hat{f}_{D_t}|\yb_t = [\hat{f}_t(x)]_{x\in D_t}$, $\mu_{D_t|\yb} = [\mu_t(x)]_{x\in D_t}$, $K_{D_t,D_t|\yb}=[k_t(x,x')]_{x,x'\in D_t}$, and $(\cdot)^{\frac{1}{2}}$ denotes a matrix square root, such as a Cholesky factor which incurs an $O(N_t^3)$ computational cost. The scaling parameter $\alpha_t\in \Rr$ is used to ensure sufficient exploration as it will become clear in the analysis. 

To improve the efficiency of sampling, a classical approximation \citep{Hernandez2014features} is to rely on 
kernel decomposition (see Sec. \ref{Subsec:feature}). In particular, a sample $\hat{f}$ can be expressed as a weighted sum of eigenfunctions
$
\hat{f}(x) =  \sum_{j=1}^\infty \sqrt{\lambda_j} w_j \phi_j(x),
$
where the weights $w_j$ are independent normal $\Nc(0, 1)$ random variables. 
Then, with decaying eigenvalues, an approximate sample $\hat{f}_M\overset{\mathrm{d}}{\approx}\hat{f}$ can be drawn from the GP using an $M-$truncated feature map as
$
\hat{f}_M(x) 
=\sum_{j=1}^M \sqrt{\lambda_j} w_j \phi_j(x).
$ 


\citet{wilson2020efficientsampling} introduced a sampling method based on decomposing the posterior sample into a truncated feature representation of the prior and an SVGP update which improves the computational cost of drawing a sample. We build on this method to introduce two sampling rules for S-GP-TS.
Our first TS rule is referred to as \emph{Decoupled Sampling with Inducing Points}:
\begin{eqnarray}\label{SampIndPoints}
\tilde{f}_{t}(x)
=\sum_{j=1}^M \alpha_t\sqrt{\lambda_j}w_j\phi_j(x) + \sum_{j=1}^{m_t}  v_{t,j} k(x,z_j),
\end{eqnarray}
where the coefficients $v_{t,j} = [K^{-1}_{\Zb_t, \Zb_t}(\alpha_t(\ub_t - \mb_t) +\mb_t - \alpha_t\Phib_{m_t,M} \Lambda_M^{\frac{1}{2}}\wb_M)]_j$, $\Phib_{m_t,M} = [\phib_M(z_1),..., \phib_M(z_{m_t})]^{\TP}$ ($m_t\times M$) and $\wb_{M} = [w_1,...,w_{M}]^{\TP}$ where $w_i$ are drawn i.i.d from $\Nc(0,1)$, for $i=1,2\dots,M$. (\ref{SampIndPoints}) is a modification of the sampling rule of~\citet{wilson2020efficientsampling} where we have scaled
the covariance of the sample with $\alpha^2_t$ (to be specified later) without changing the mean. The scaling is necessary for TS methods to ensure sufficient exploration based on anti-concentration of Gaussian distributions. 
If we set $\alpha_t = 1$, we retrieve the exact sampling rule of~\citet{wilson2020efficientsampling}. 

For SVGP models with inducing features, the decoupled sampling rule becomes: 
\begin{eqnarray}\label{SampIndVariables}
\tilde{f}_{t}(x)
=
\sum_{j=1}^M \alpha_t\sqrt{\lambda_j}w_j\phi_j(x) + \sum_{j=1}^{m_t} v_{t,j} {\lambda_j}\phi_j(x),
\end{eqnarray}
where $v_{t,j} = [\Lambda_{m_t}^{-1}(\alpha_t(\ub_t -\mb_t)+\mb_t -\alpha_t\Lambda_{m_t}^{\frac{1}{2}} \wb_{m_t})]_j$.
The computational complexity of both sampling rules introduced above is $O(tm_t^2 + N_t M)$ per step $t$\footnote{The computational cost also comprises an $O(M m_t + N_t m_t)$ term which is dominated by $O(N_t M)$. This will become clear later when the values of $N_t, m_t$ and $M$ are specified. }. 

\section{Regret Analysis}\label{Sec:Analysis}

In this section, we present the main contribution of our work that is the theoretical analysis of S-GP-TS methods. First, we establish an upper bound on the regret of any approximate GP model (Theorem~\ref{The:RegretBoundsApproximate}) based on the quality of approximations parameterized in Assumptions~\ref{Ass1} and~\ref{Ass2}. We then discuss the consequences of Theorem~\ref{The:RegretBoundsApproximate} for the regret bounds and the computational complexity of S-GP-TS methods based on SVGP and the decoupled sampling rules~\eqref{SampIndPoints} and~\eqref{SampIndVariables}.


\subsection{The Analysis of Vanilla GP-TS} 

\citeauthor{Chowdhury2017bandit} proved that, with probability at least $1-\delta$, $|f(x) - \mu_{t}(x)|\le {u}_t\sigma_t(x)$, where ${u}_t = \left( B+R\sqrt{2(\gamma_{t}+1+\log(1/\delta))} \right)$ and $\gamma_t$ is the \textit{maximal information gain}:
\begin{equation}
    \gamma_t = \max_{\Xb_t\subset\X}I(\yb_t, \hat{\bm{f}}_t),
\end{equation}
where $I(\yb_t, \hat{\bm{f}}_t) = \frac{1}{2}\log( \det(\Ib_t+\tau^{-1}\Kb_{\Xb_t,\Xb_t}))$
denotes the mutual information between observations $\yb_t$ and the underlying model values $\hat{\bm{f}}_t=[\hat{f}(x_1),..., \hat{f}(x_t)]$. The maximal information gain can itself be bounded for a specific kernel, e.g. $\gamma_t\le O(\log(t)^{d+1})$ for the SE kernel and $\gamma_t\le O(t^{d(d+1)/(2\nu+d(d+1))}\log(t))$ for the Mat{\'e}rn kernel \citep{srinivas2010gaussian}. 

Based on this concentration inequality, \citeauthor{Chowdhury2017bandit} showed that the regret of GP-TS scales with the cumulative uncertainty at the observation points measured by the standard deviation: $O(\sum_{t=1}^T {u}_t\sigma_{t-1}(x_t))$.
Furthermore, \citeauthor{srinivas2010gaussian} showed that $\sum_{i=1}^t\sigma_{i-1}^2(x_i)\le \gamma_t$. Using this result and applying Cauchy-Schwarz inequality to $O(\sum_{t=1}^T {u}_t\sigma_{t-1}(x_t))$,  \citeauthor{Chowdhury2017bandit} proved that
$
R( T; \text{vanilla GP-TS}) = {O}\left(\gamma_T\sqrt{T\log(T)}\right)
$. 

\subsection{Regret Bounds Based on the Quality of Approximations}\label{Sec:RegSec}

Consider a TS algorithm using an approximate GP model to solve the sequential optimization problem. In particular, assume an approximate model with kernel $\tilde{k}$ and mean $\tilde{\mu}$ is provided such that $\tilde{k}_t$, $\tilde{\sigma}_t$ and $\tilde{\mu}_t$ are, respectively, approximations of ${k}_t$, ${\sigma}_t$ and ${\mu}_t$ in the surrogate GP model (defined in~\eqref{GPt}). At each time $t$, a batch of samples $\{\tilde{f}_{t,b}\}_{b\in [B]}$ is drawn from a GP with mean $\tilde{\mu}_{t-1}$ and the scaled covariance $\alpha_t^2\tilde{k}_{t-1}$. The next observation point $x_t$ is selected as the maximizer of $\tilde{f}_t$ over a discretization $D_t$ of the search space.
We start our analysis by making two assumptions on the \emph{quality} of approximations $\tilde{\mu}_t$, $\tilde{\sigma}_t$ of the posterior mean and the standard deviation. This parameterization is agnostic to the particular sampling rule (governing $\tilde{\mu}_t$ and $\tilde{\sigma}_t$)
and
provides valuable intuition that can be applied to any approximate method. I.e., for the results presented in this section, $\tilde{\mu}_t$, $\tilde{\sigma}_t$ may correspond to any approximate model and are not limited to the decoupled sampling rules.

\begin{assumption}[{quality} of the approximate standard deviation]\label{Ass1}
For the approximate $\tilde{\sigma}_{t}$, the exact $\sigma_t$, and for all $x\in \X$,
\[
\frac{1}{\underline{a}_t}\sigma_t(x)-\epsilon_t\le \tilde{\sigma}_{t}(x) \le \bar{a}_t\sigma_t(x)+\epsilon_t,
\]
where $1\le\underline{a}_t\le\underline{a}$, $1\le\bar{a}_t\le\bar{a}$ for all $t\ge1$ and some constants $\underline{a}, \bar{a}\in\Rr$, and $0\le\epsilon_t\le \epsilon$ for all $t\ge1$ and some small constant $\epsilon\in \Rr$.
\end{assumption}

\begin{assumption}[{quality} of the approximate prediction]\label{Ass2}
For the approximate $\tilde{\mu}_t$, the exact $\mu_t$ and $\sigma_t$, and for all $x\in \X$,
\[
|\tilde{\mu}_t(x) - \mu_t(x)| \le c_t {\sigma}_t(x),
\]
where $0\le c_t\le c$ for all $t\ge1$ and some constant $c\in \Rr$.
\end{assumption}
The following Lemma establishes a concentration inequality for the approximate statistics as a direct result of Theorem 2 of~\cite{Chowdhury2017bandit}.

\begin{lemma}\label{Lemma:ConIneqApprox}
Under Assumptions~\ref{Ass1}~and~\ref{Ass2}, with probability at least $1-\delta$,
$
|f(x)-\tilde{\mu}_t(x)| \le  \u_t (\tilde{\sigma}_t(x)+\epsilon_t)
$, where $\u_t(\delta)=\underline{a}_t\left( B+R\sqrt{2(\gamma_{t}+1+\log(1/\delta))} +c_t\right)$.

\end{lemma}

Proof is provided in Appendix~A. 
Following \citeauthor{Chowdhury2017bandit, srinivas2010gaussian}, we also assume that $D_t$ is designed in a way that  $|f(x)-f(\mathtt{x}^{(t)})|\le 1/t^2$ for all $x\in\X$, where $\mathtt{x}^{(t)}=\argmin_{x'\in D_t}||x-x'||$ is the closest point (in Euclidean norm) to $x$ in $D_t$. 
The size of this discretization satisfies $N_t\le C(d,B) t^{2d}$ where $C(d,B)$ is independent of $t$ (\citeauthor{Chowdhury2017bandit,srinivas2010gaussian}). 
We are now in a position to present the regret bounds based on the quality of the approximations.

\begin{theorem}\label{The:RegretBoundsApproximate}
Consider S-GP-TS with $\alpha_t = 2\u_{t}(1/(t^2))$ using an approximate GP model satisfying Assumptions~\ref{Ass1}~and~\ref{Ass2}. Assume $||f||_{H_k}\le B$ and the observation noise is $R-$sub-Gaussian. The regret (defined in~\eqref{eq:regdef}) satisfies
\begin{eqnarray}\nn
R(T;\text{S-GP-TS}) &\le& 
30\bar{a}\beta_T\sqrt{4\gamma_T(T+2)}+
(31\beta_T+\alpha_T)\epsilon T + 15B+2,
\end{eqnarray}
where $\beta_t = \alpha_t(b_t+\frac{1}{2})$ with $b_t = \sqrt{2\log(N_t t^2)}$. 
\end{theorem}
The expression can be simplified as
$R(T;\text{S-GP-TS})= {O}\left( \underline{a}\bar{a}\gamma_T\sqrt{ T\log(T)} + \underline{a}\epsilon T\sqrt{\gamma_T\log(T)}  \right)$,
where it is easier to see the scaling of regret with respect to the parameterization of the quality of the approximations. The regret bound scales with the product of the ratios $\underline{a}$ and $\bar{a}$. It also comprises an additive term depending on the additive approximation error in the standard deviation.

Parameter $B$ is used to properly tune the scaling $\alpha_t$. If it is not known in advance, standard guess-and-double techniques apply~(the same as in GP-UCB, see \citeauthor{srinivas2010gaussian, Calandriello2019Adaptive})

\subsection{Approximation Quality of the Decomposed Sampling Rule}

The quality of the approximation is characterized using the spectral properties of the GP kernel. Let us define $\delta_M =\sum_{i=M+1}^\infty\lambda_i\psi^{2}_i$ where $\psi_i =\max_{x\in\X} \phi_i(x)$.
With decaying eigenvalues, including sufficient eigenfunctions in the feature representation results in a small $\delta_M$. 
In addition, \citeauthor{Burt2019Rates} showed that, for SVGP, a sufficient number of inducing variables ensures that the Kullback–Leibler (KL) divergence between the approximate and the true posterior distributions diminishes. Consequently, the approximate posterior mean and the approximate posterior variance converge to the true ones. Building on this result, we prove the following proposition on the quality of approximations. 
\begin{proposition}\label{Prop:paramters}
For S-GP-TS based on sampling rule~\eqref{SampIndPoints} with $\alpha_t=1$ and an SVGP using an $\epsilon_0$ close k-DPP for selecting $\Zb_t$, with probability at least $1-\delta$, Assumptions~\ref{Ass1} and~\ref{Ass2} hold with parameters $c_t =  \sqrt{\kappa_t}$, $\underline{a}_t = \frac{1}{\sqrt{1-\sqrt{3\kappa_t}}}$, $\bar{a}_t = \sqrt{1+\sqrt{3\kappa_t}}$, and $\epsilon_t = \sqrt{C_1m_t\delta_M}$,
where $C_1$ is a constant specified in Appendix~C and $\kappa_t = \frac{2t(m_t+1)\delta_{m_t}}{\delta\tau} + \frac{4t\epsilon_0}{\delta\tau}$.

For S-GP-TS based on sampling rule~\eqref{SampIndVariables} with $\alpha_t=1$, Assumptions~\ref{Ass1} and~\ref{Ass2} hold with parameters $c_t =  \sqrt{\kappa_t}$, $\underline{a}_t = \frac{1}{\sqrt{1-\sqrt{3\kappa_t}}}$, $\bar{a}_t = \sqrt{1+\sqrt{3\kappa_t}}$, and $\epsilon_t = \sqrt{C_1m_t\delta_M}$,
where $C_1$ is the same constant as above and $\kappa_t = \frac{2t\delta_{m_t}}{\tau}$.
\end{proposition}

The results in~\citeauthor{Burt2019Rates} do not directly apply to our setting for two reasons. First, we use the decoupled sampling rules (\eqref{SampIndPoints} and~\eqref{SampIndVariables}) which introduce additional error in approximate variance compared to SVGP (while keeping the prediction the same as in SVGP). This additional error in the approximate variance in particular makes the analysis of S-GP-TS more challenging. Second, \citeauthor{Burt2019Rates} build their convergence results on the assumption that the observation points $x_t$ are drawn from a prefixed distribution which is not the case in S-GP-TS, where $x_t$ are selected according to an experimental design method. A detailed proof of Proposition~\ref{Prop:paramters} is provided in Appendix~C. 

\begin{table*}[h]
\centering
\begin{tabular}{  c c c} 
 \Xhline{2\arrayrulewidth}
   Kernel&  Sampling rule~\eqref{SampIndPoints}&Sampling rule~\eqref{SampIndVariables}\\
 \hline
  &&\\
   Mat{\'e}rn ($\nu$) & $O\left(N_TT^{\frac{4\nu^2+d^2}{2(2\nu-d)\nu}}+
 T^2\min\{T^{\frac{4d}{2\nu-d}}, T^2\} \right)$& 
  $O\left( N_TT^{\frac{(2\nu+d)^2-2\nu d}{4\nu^2}} + T^{\frac{2\nu+d}{\nu}}  \right)$ \\ 
   SE  &  $O\left(N_TT\log^{d}(T)+T^2\log^{2d}(T)\right)$&  $O\left(N_TT\log^{d}(T)+T^2\log^{2d}(T)\right)$\\ 
    &&\\
 \hline
 &&\\
\end{tabular}\caption{  The computational complexity of S-GP-TS using Decoupled Sampling with Inducing Points~\eqref{SampIndPoints} and Decomposed Sampling with Inducing Features~\eqref{SampIndVariables} with Mat{\'e}rn (with smoothness parameter $\nu$) and SE kernels.}\label{Table:CC}
\end{table*}

\subsection{Application of Regret Bounds to Mat{\'e}rn and SE Kernels}

In this section, we show the application of Theorem~\ref{The:RegretBoundsApproximate} to the Matern and SE kernels, based on their spectrum properties. In the case of a Mat{\'e}rn kernel with smoothness parameter $\nu>\frac{d}{2}$ it is known that $\lambda_j  =O(j^{-\frac{2\nu+d}{d}})$~\citep{MaternEigenvaluessantin2016}. In the case of SE kernel, it is known that $\lambda_j  = O(\exp(-j^{\frac{1}{d}}))$~\citep{SEEigenvalues}. Also, see~\citet{Gabriel2020practicalfeature} which gave closed form expression of their eigenvalue-eigenfunction pairs on hypercubes.
With these bounds on the spectrum of the kernels, Theorem~\ref{The:RegretBoundsApproximate} and Proposition~\ref{Prop:paramters} result in the following theorem.
\begin{theorem}\label{The:MatSE}
Assume $||f||\le B$ and the observation noise is sub-Gaussian. 
Consider S-GP-TS under following different cases:
\begin{itemize}
    \item[$1$.] Mat{\'e}rn kernel with parameter $\nu$ under sampling rule~\eqref{SampIndPoints} with $m_t \sim T^{\frac{2d}{2\nu-d}}$, $M\sim T^{\frac{(2\nu+d)d}{2(2\nu-d)\nu}}$,
    \item[$2$.] Mat{\'e}rn kernel with parameter $\nu$ under sampling rule~\eqref{SampIndVariables} with $m_t \sim T^{\frac{d}{2\nu}}$, $M\sim T^{\frac{(2\nu+d)d}{4\nu^2}}$,
    \item[$3$.] SE kernel under sampling rule~\eqref{SampIndPoints} with $m_t,M\sim (\log(T))^{d}$,
    \item[$4$.] SE kernel under sampling rule~\eqref{SampIndVariables} with $m_t,M\sim (\log(T))^{d}$.
\end{itemize}
Under all cases above, we have 
$
R(T; S-GP-TS) = O(\gamma_T\sqrt{T\log(T)}).
$

\end{theorem}


To prove Theorem~\ref{The:MatSE}, it is shown that the algorithmic parameters $M$ and $m_t$ are selected large enough such that approximation parameters $\underline{a}, \bar{a}, c, \epsilon$ in Assumptions~\ref{Ass1} and~\ref{Ass2} are sufficiently small. The relation between the algorithmic parameters and the approximation parameters and $m_t$ is given in Proposition~\ref{Prop:paramters}. The regret bound follows form Theorem~\ref{The:RegretBoundsApproximate}. A detailed proof is provided in Appendix~.

 Table~\ref{Table:CC} reports the computational costs, $O\left((M+m_T)N_T + m_T^2T^2\right) $, under the four cases, and shows the improvements over the $O(N_T^3T+T^4)$ computational cost of the vanilla GP-TS. For the Mat{\'e}rn kernel, under sampling rule~\eqref{SampIndPoints}, in order for $m_t$ to grow slower than $t$, $\nu$ is required to be sufficiently larger than $\frac{d}{2}$.

\section{Conclusion}\label{Sec:Disc}

\bibliography{references.bib}
\bibliographystyle{abbrvnat}

\newpage

\title{Scalable Thompson Sampling using Sparse Gaussian Process Models: \\ 
Supplementary Materials}

%
%
\vspace{-10em}
In this section, we provide detailed proofs for Theorem~\ref{The:RegretBoundsApproximate}, Proposition~\ref{Prop:paramters} and Theorem~\ref{The:MatSE}.
%
%

\subsection*{Proof of Theorem~\ref{The:RegretBoundsApproximate}}

We build on the analysis of GP-TS in~\cite{Chowdhury2017bandit} to prove the regret bounds for S-GP-TS. Although the proofs are to some extent similar, nonetheless the analysis of standard GP-TS does not extend to S-GP-TS. This proof characterizes the behavior of the upper bound on regret in terms of the approximation constants, namely $\underline{a}, \bar{a}, {c}$ and $\epsilon$. 
A notable difference is that the additive approximation error in the posterior standard deviation ($\epsilon_t$) can cause under-exploration which is an issue the analysis of exact GP-TS cannot address. 

We first focus on the instantaneous regret at each time $t$ within the discrete set, $f({\mathtt{x}^*}^{(t)}) - f(x_t)$. Recall ${\mathtt{x}^*}^{(t)} \triangleq \text{argmin}_{x'\in D_t}||x^*-x'||$. It is then easy to upper bound the cumulative regret as the cumulative value of $f({\mathtt{x}^*}^{(t)}) - f(x_t)+\frac{1}{t^2}$ as a result of the property of the discretization that ensures $f(x^*)-f({\mathtt{x}^*}^{(t)})\le \frac{1}{t^2}$. 
For upper bounds on instantaneous regret, we start with concentration of GP samples $\tilde{f}_t$ around their predicted values as well as concentration of the prediction around the objective function. We then consider the anti-concentration around the optimum point. The necessary anti-concentration may fail due to approximation error in the standard deviation around the optimum point. We thus consider two cases of low and sufficiently high standard deviation at ${\mathtt{x}^*}^{(t)}$ separately. While a low standard deviation implies good prediction at ${\mathtt{x}^*}^{(t)}$, a sufficiently high standard deviation guarantees sufficient exploration. We use these results to upper bound the instantaneous regret at each time $t$ with uncertainties measured by the standard deviation. 

\textbf{Concentration events $\Ec_t$ and $\tilde{\Ec}_t$:}

\textbf{Define $\Ec_t$} as the event that at time $t$, for all $x\in D_t$, 
$
|f(x) - \tilde{\mu}_{t-1}(x)| \le \frac{1}{2}\alpha_t(\tilde{\sigma}_{t-1}(x)+\epsilon_t)
$.
Recall $\alpha_t = 2\u_t(1/(t^2))$. Applying lemma~\ref{Lemma:ConIneqApprox}, we have $\Pr[\Ec_t]\ge 1-\frac{1}{t^2}$.

\textbf{Define $\tilde{\Ec}_t$} as the event that for all $x\in D_t$, 
$
|\tilde{f}_t(x) - \tilde{\mu}_{t-1}(x)| \le  \alpha_tb_t \tilde{\sigma}_{t-1}(x)
$
where $b_t = \sqrt{2\ln(N_t t^2)}$. We have 
$
\Pr[\tilde{\Ec}_t] \ge 1-\frac{1}{t^2}.
$

\emph{Proof.}
For a fixed $x\in D_t$,
\begin{eqnarray}\nn
\Pr\left[|\tilde{f}_t(x) - \tilde{\mu}_{t-1}(x)| >  \alpha_{t}b_t \tilde{\sigma}_{t-1}(x)\right] <\exp(-\frac{b_t^2}{2}) = \frac{1}{N_tt^2}.
\end{eqnarray}
The inequality holds because of the following bound on the CDF of a normal random variable $1-\text{CDF}_{\Nc(0,1)}(c)\le \frac{1}{2}\exp(-\frac{c^2}{2})$ and the observation that $\frac{\tilde{f}_t(x) - \tilde{\mu}_{t-1}(x)}{\alpha_t \tilde{\sigma}_{t-1}(x)}$ has a normal distribution. Applying union bound we get
$
\Pr[\bar{\tilde{\Ec}}_t] \le \frac{1}{t^2}
$
which gives us the bound on probability of $\tilde{\Ec}_t$.  ~~~~~ $\square$

We thus proved $\Ec_t$ and $\tilde{\Ec}_t$ are high probability events. This will facilitate the proof by conditioning on $\Ec_t$ and $\tilde{\Ec}_t$. 
Also notice that when both $\Ec_t$ and $\tilde{\Ec}_t$ hold true, we have, for all $x\in D_t$
\begin{eqnarray}\label{eq:beta}
|\tilde{f}_t(x) - f(x)| \le \beta_t\tilde{\sigma}_{t-1}(x) + \frac{1}{2}\alpha_t\epsilon_t
\end{eqnarray}
where $\beta_t = \alpha_t(b_t+\frac{1}{2})$.

\paragraph{Anti Concentration Bounds.} It is standard in the analysis of TS methods to prove sufficient exploration by building on anti-concentration bounds. For this purpose we use the following bound on the CDF of a normal distribution: $1-\text{CDF}_{\Nc(0,1)}(c)\ge \frac{\exp(-{c^2})}{4c\sqrt{\pi}}$. The underestimation of the posterior standard deviation at the optimum point however might result in an under exploration. On the other hand, a low standard deviation at the optimum point implies a low prediction error. We use this observation in our regret analysis by considering the two cases separately. Specifically, the regret $ f({\mathtt{x}^*}^{(t)}) - f(x_t)$ at each time $t$ is bounded differently under the conditions: I. $\tilde{\sigma}_{t-1}({\mathtt{x}^*}^{(t)})\le \epsilon_t$ and II. $\tilde{\sigma}_{t-1}({\mathtt{x}^*}^{(t)})> \epsilon_t$. \\

\textbf{Under Condition} I ($\tilde{\sigma}_{t-1}({\mathtt{x}^*}^{(t)})\le \epsilon_t$), when both $\Ec_t$ and $\tilde{\Ec}_t$ hold true, we have
\begin{eqnarray}\nn
&&\hspace{-4em}f({\mathtt{x}^*}^{(t)}) - f(x_t)\\\nn
&\le& \tilde{f}_t({\mathtt{x}^*}^{(t)}) + \beta_t\tilde{\sigma}_{t-1}({\mathtt{x}^*}^{(t)})+ \frac{1}{2}\alpha_t\epsilon_t \\\nn
&&~~~~~- \tilde{f}_{t}(x_t) + \beta_t\tilde{\sigma}_{t-1}(x_t)+ \frac{1}{2}\alpha_t\epsilon_t~\hspace{7em}\text{by~\eqref{eq:beta},}\\\label{eq:boundstarsig}
&\le&  \beta_t\tilde{\sigma}_{t-1}({\mathtt{x}^*}^{(t)})+ \beta_t\tilde{\sigma}_{t-1}(x_t)+\alpha_t\epsilon_t~\hspace{1.2em}\text{by~the selection rule of TS,}\\\nn
&\le&\beta_t\tilde{\sigma}_{t-1}(x_t) +    (\beta_t +\alpha_t)\epsilon_t~\hspace{9em}~~\text{by Condition I}.
\end{eqnarray}
that upper bounds the instantaneous regret at time $t$ by a factor of approximate standard deviation up to an additive term caused by approximation error. Since $f({\mathtt{x}^*}^{(t)}) - f(x_t)\le 2B$, under Condition I,
\begin{eqnarray}\label{Ass1simplereg}
\E[f({\mathtt{x}^*}^{(t)}) - f(x_t)] \le  \beta_t\tilde{\sigma}_{t-1}(x_t) + (\beta_t+\alpha_t)\epsilon_t +\frac{4B}{t^2}.
\end{eqnarray}
where the inequality holds by $\Pr[\bar{{\Ec}}_t ~\text{or}~\bar{\tilde{\Ec}}_t  ]\le \frac{2}{t^2}$.

\textbf{Under Condition} II ($\tilde{\sigma}_{t-1}({\mathtt{x}^*}^{(t)})> \epsilon_t$), we can show sufficient exploration by anti-concentration at the optimum point. In particular under Condition II, if $\Ec_t$ holds true, we have \begin{eqnarray}\label{eq:plow}
\Pr[\tilde{f}_t({\mathtt{x}^*}^{(t)})>f({\mathtt{x}^*}^{(t)})] \ge p,
\end{eqnarray}
where $p= \frac{1}{4\sqrt{\pi}}$.

\emph{Proof.} 
Applying the anti-concentration of normal distribution 
\begin{eqnarray}\nn
\Pr[\tilde{f}_t({\mathtt{x}^*}^{(t)})>f({\mathtt{x}^*}^{(t)})] &=& \Pr\left[ \frac{\tilde{f}_t({\mathtt{x}^*}^{(t)})-\tilde{\mu}_{t-1}({\mathtt{x}^*}^{(t)})}{\alpha_{t}\tilde{\sigma}_{t-1}({\mathtt{x}^*}^{(t)})}> \frac{f({\mathtt{x}^*}^{(t)})-\tilde{\mu}_{t-1}({\mathtt{x}^*}^{(t)})}{\alpha_{t}\tilde{\sigma}_{t-1}({\mathtt{x}^*}^{(t)})}\right]\\\nn
&\ge& p.
\end{eqnarray}
As a result of the observation that the right hand side of the inequality inside the probability argument is upper bounded by $1$:
\begin{eqnarray}\nn
\frac{f({\mathtt{x}^*}^{(t)})-\tilde{\mu}_{t-1}({\mathtt{x}^*}^{(t)})}{\alpha_{t}\tilde{\sigma}_{t-1}({\mathtt{x}^*}^{(t)})} &\le& \frac{  \frac{1}{2}\alpha_t\tilde{\sigma}_{t-1}({\mathtt{x}^*}^{(t)})+ \frac{1}{2}\alpha_t\epsilon_t }{\alpha_t\tilde{\sigma}_{t-1}({\mathtt{x}^*}^{(t)})}~~~~\text{By}~\Ec_t\\\nn
&\le& 1.~~~\hspace{13em}~\text{By Condition II}~~~~~\square
\end{eqnarray}

\paragraph{Sufficiently Explored Points.} Let $\Sc_t$ denote the set of sufficiently explored points which are unlikely to be selected by S-GP-TS if $\tilde{f}_t({\mathtt{x}^*}^{(t)})$ is higher than $f({\mathtt{x}^*}^{(t)})$. Specifically, we use the notation
\begin{eqnarray}
\Sc_t = \{x\in D_t:  f(x)+\beta_t\tilde{\sigma}_{t-1}(x)+ \frac{1}{2}\alpha_t\epsilon_t \le f({\mathtt{x}^*}^{(t)})\}. 
\end{eqnarray}  
Recall $\beta_t = \alpha_t(b_t+\frac{1}{2})$. In addition, we define
\begin{eqnarray}
\bar{x}_t ={\text{argmin}}_{x\in D_t\setminus \Sc_t}\tilde{\sigma}_{t-1}(x).
\end{eqnarray}

We showed in equation~\eqref{eq:boundstarsig} that the instantaneous regret can be upper bounded by the sum of standard deviations at $x_t$ and ${\mathtt{x}^*}^{(t)}$. The standard method based on information gain can be used to bound the cumulative standard deviations at $x_t$. 
This is not sufficient however because the cumulative standard deviations at ${\mathtt{x}^*}^{(t)}$ does not converge unless there is sufficient exploration around $x^*$. To address this, we use $\bar{x}_t$ as an intermediary to be able to upper bound the instantaneous regret by a factor of $\tilde{\sigma}_{t-1}(x_t)$ through the following lemma. 
\begin{lemma}\label{Lemma:xbar}
Under Condition II, for $t\ge\sqrt{\frac{2}{p}}$, if $\Ec_t$ holds true
\begin{eqnarray}
\tilde{\sigma}_{t-1}(\bar{x}_t)\le \frac{2}{p} \E[\tilde{\sigma}_{t-1}(x_t)],
\end{eqnarray}
where the expectation is taken with respect to the randomness in the sample $\tilde{f}_t$. 
\end{lemma}
\begin{proof}[Proof of Lemma~\ref{Lemma:xbar}]
First notice that when both $\Ec_t$ and $\tilde{\Ec_t}$ hold true, for all $x\in \Sc_t$
\begin{eqnarray}\nn
\tilde{f}_t(x)&\le& f(x) + \beta_t \tilde{\sigma}_{t-1}(x)+(\alpha_t-1)\epsilon_t~~~\hspace{2em} \text{by~\eqref{eq:beta}}\\\label{eq:ftildes1}
&\le& f({\mathtt{x}^*}^{(t)}), ~~~~\hspace{5em}~\text{by definition of}~\Sc_t.
\end{eqnarray}
Also, if $\tilde{f}_t({\mathtt{x}^*}^{(t)})>\tilde{f}_t(x), \forall x \in S_t$, the rule of selection in TS ensures $x_t\in D_t\setminus \Sc_t$. So we have
\begin{eqnarray}\nn
\Pr[x_t \in D_t\setminus \Sc_t] &\ge& \Pr[\tilde{f}_t({\mathtt{x}^*}^{(t)})>\tilde{f}_t(x), \forall x \in S_t]\\\nn
&\ge& \Pr[\tilde{f}_t({\mathtt{x}^*}^{(t)})>\tilde{f}_t(x), \forall x \in S_t, \tilde{\Ec_t}] -
\Pr[\bar{\tilde{\Ec}}_t]
\\\nn
&\ge& \Pr[\tilde{f}_t({\mathtt{x}^*}^{(t)})>f({\mathtt{x}^*}^{(t)})] -
\Pr[\bar{\tilde{\Ec}}_t]~~~\hspace{2em}~\text{by \eqref{eq:ftildes1}}~\\\nn
&\ge& p-\frac{1}{t^2} ~~~\hspace{2em}~\text{by \eqref{eq:plow}}
\\\nn
&\ge&
\frac{p}{2}, ~~~\hspace{2em}~\text{for}~t\ge \sqrt{2/p}.
\end{eqnarray}
Finally, we have
\begin{eqnarray}
\E[\tilde{\sigma}_{t-1}(x_t)] &\ge& \E\left[\tilde{\sigma}_{t-1}(x_t)\bigg|x_t\in D_t\setminus S_t\right] \Pr[x_t\in D_t\setminus S_t]\\\nn
&\ge& \frac{p\tilde{\sigma}_{t-1}(\bar{x}_t)}{2},
\end{eqnarray}
where the expectation is taken with respect to the randomness in the sample $\tilde{f}_t$ at time $t$. 
\end{proof}

Now we are ready to bound the simple regret under Condition II using $\bar{x}_t$ as an intermediary. 
Under Condition II, when both $\Ec_t$ and $\tilde{\Ec}_t$ hold true, 
\begin{eqnarray}\nn
f({\mathtt{x}^*}^{(t)}) - f(x_t) &=& f({\mathtt{x}^*}^{(t)}) - f(\bar{x}_t)+f(\bar{x}_t) - f(x_t) \\\nn
&\le& \beta_t\tilde{\sigma}_{t-1}(\bar{x}_t)+ \frac{1}{2}\alpha_t\epsilon_t
+f(\bar{x}_t) - f(x_t)~~~\text{by definition of}~\Sc_t\\\nn
&\le & \beta_t\tilde{\sigma}_{t-1}(\bar{x}_t)+ \frac{1}{2}\alpha_t\epsilon_t\\\nn
&&~~~~~+ \tilde{f}_t(\bar{x}_t) + \beta_t\tilde{\sigma}_{t-1}(\bar{x}_t) - \tilde{f}_t(x_t) +\beta_t\tilde{\sigma}_{t-1}(x_t)+\alpha_t\epsilon_t~~~\text{by~\eqref{eq:beta}}\\\nn
&\le& \beta_t(2\tilde{\sigma}_{t-1}(\bar{x}_t)+\tilde{\sigma}_{t-1}(x_t) ) + \frac{3}{2}\alpha_t\epsilon_t, ~~~\text{by the rule of selection in TS}.
\end{eqnarray}

Thus, since $f(x^*)-f(x_t)\le 2B$, under Condition II, for $t\ge \sqrt{\frac{2}{p}}$
\begin{eqnarray}
\E[f({\mathtt{x}^*}^{(t)})-f(x_t)]\le \frac{(4+p)\beta_t}{p}\E[\tilde{\sigma}_{t-1}(x_t)] +\frac{3}{2}\alpha_t\epsilon_t + \frac{4B}{t^2}
\end{eqnarray}
where we used Lemma~\ref{Lemma:xbar} and $\Pr[\bar{{\Ec}}_t ~\text{or}~\bar{\tilde{\Ec}}_t  ]\le \frac{2}{t^2}$.

\paragraph{Upper bound on regret.} From the upper bounds on instantaneous regret under Condition I and Condition II we conclude that, for $t\ge \sqrt{\frac{2}{p}}$
\begin{eqnarray}
\E[f({\mathtt{x}^*}^{(t)})-f(x_t)]&\le& \max \bigg\{ \beta_t\tilde{\sigma}_{t-1}(x_t) + (\beta_t+\alpha_t)\epsilon_t +\frac{4B}{t^2}, \\\nn &&~~~~~\frac{(4+p)\beta_t}{p}\E[\tilde{\sigma}_{t-1}(x_t)] +\frac{3}{2}\alpha_t\epsilon_t + \frac{4B}{t^2}  \bigg\}\\\nn
&\le& \frac{(4+p)\beta_t}{p}\E[\tilde{\sigma}_{t-1}(x_t)] + (\beta_t+\alpha_t)\epsilon_t+ \frac{4B}{t^2}.
\end{eqnarray}
We can now upper bound the cumulative regret. Noticing $\lceil\sqrt{\frac{2}{p}}\rceil = 4$. 
\begin{eqnarray}\nn
R(T;S-GP-TS) &=& \sum_{t=1}^T \E[f(x^*)-f(x_t)]\\\nn
&+& \sum_{t=1}^{4} \E[f(x^*)-f(x_t)]+ \sum_{t=5}^T\E[f(x^*)-f(x_t)]\\\nn
&\le&8B +  \sum_{t=5}^T(\E[f({\mathtt{x}^*}^{(t)})-f(x_t)]+\frac{1}{t^2})\\\nn
&\le&8B +  \sum_{t=5}^T \left(\frac{(4+p)\beta_t}{p}\E[\tilde{\sigma}_{t-1}(x_t)] + (\beta_t+\alpha_t)\epsilon_t+ \frac{4B+1}{t^2}\right)\\\nn
&\le& 8B + \frac{\pi^2(4B+1)}{6} + \frac{(4+p)\beta_T}{p}\sum_{t=1}^T\E[\tilde{\sigma}_{t-1}(x_t)] + (\beta_T+\alpha_T)\sum_{t=1}^T\epsilon_t\\\nn
&\le& 15B+2 + 30\beta_T\sum_{t=1}^T(\bar{a}\E[\sigma_{t-1}(x_t)]+\epsilon_t)
+ (\beta_T+\alpha_T){\epsilon}T\\\nn
&\le& 15B+2 + 30\bar{a}\beta_T\sum_{t=1}^T\E[\sigma_{t-1}(x_t)]+ 30\beta_T\epsilon T
+ (\beta_T+\alpha_T){\epsilon}T\\\nn
&\le& 15B+2
+ 30\bar{a}\beta_T\sum_{t=1}^T\E[\sigma_{t-1}(x_t)]+
(31\beta_T+\alpha_T)\epsilon T.
\end{eqnarray}

We simplified the expressions by $\frac{4+p}{p}\le 30$, $\frac{4\pi^2}{6}\le 7$ and $\frac{\pi^2}{6}\le 2$. We also use the standard information gain technique to upper bound $\sum_{t=1}^T\E[\sigma_{t-1}(x_t)]$. Specifically $\sum_{t=1}^T\E[\sigma_{t-1}(x_t)]\le \sqrt{4\gamma_T(T+2)}$~(Lemma 4 in \cite{Chowdhury2017bandit}). We thus have
\begin{eqnarray}
R(T;\text{S-GP-TS}) \le 
 30\bar{a}\beta_T\sqrt{4\gamma_T(T+2)}+
(31\beta_T+\alpha_T)\epsilon T + 15B+2
\end{eqnarray}

which can be simplified to
\begin{eqnarray}
R(T;\text{S-GP-TS}) = \tilde{O}\left( \underline{a}\bar{a}(1+c)\sqrt{\gamma_T T} + \underline{a}^2(1+c^2)\epsilon T  \right).
\end{eqnarray}

  $\square$

\begin{proof}[Proof of Lemma~\ref{Lemma:ConIneqApprox}]

By triangle inequality we have
\begin{eqnarray}\nn
|f(x) - \tilde{\mu}_t(x)| &\le& |f(x)-\mu_t(x)|  + |\tilde{\mu}_t(x) - \mu_t(x)| \\\nn
&\le& |f(x)-\mu_t(x)| + c_t\sigma_t(x)~~~~\text{by Assumptions~\ref{Ass2}}.
\end{eqnarray}
From Theorem 2 of~\cite{Chowdhury2017bandit}, 
with probability at least $1-\delta$,
\begin{eqnarray}\nn
f(x)- \mu_t(x)\le\left( B+R\sqrt{2(\gamma_{t}+1+\log(1/\delta))} \right)\sigma_t(x).
\end{eqnarray}

Thus,
\begin{eqnarray}\nn
|f(x)-\tilde{\mu}_t(x)| 
&\le& \left( B+R\sqrt{2(\gamma_{t}+1+\log(1/\delta))} \right) \sigma_t(x) + c_t {\sigma}_{t}(x)
\\\nn
&\le& \underline{a}_t(B+R\sqrt{\frac{2\ln(1/\delta)}{\tau}} +c_t) (\tilde{\sigma}_t(x)+\epsilon_t),
\\\nn
\end{eqnarray}
where the last inequality holds by Assumption~\ref{Ass1}.
\end{proof}


\subsection*{Proof of Proposition~\ref{Prop:paramters}}

Here, we use $\tilde{\mu}_t$ and $\tilde{\sigma}_t$ to specifically denote the approximate posterior mean and the approximate posterior standard deviations of the decomposed sampling rules~\eqref{SampIndPoints} and~\eqref{SampIndVariables} in contrast to Sec.~\ref{Sec:RegSec} where we used the notation more generally for any approximate model. We also use $\mu^{(s)}_t$ and $\sigma^{(s)}_t$ to refer to the posterior mean and the posterior standard deviation of SVGP models, and $\mu^{(w)}$ and $\sigma^{(w)}$ to refer to the priors generated form an $M-$truncated feature vector. 
For the approximate posterior mean, we have $\tilde{\mu}_t = \mu^{(s)}_t$. However the approximate posterior standard deviations $\sigma^{(s)}$ and $\tilde{\sigma}$ are not the same.

By triangle inequality we have
\begin{eqnarray}
|\tilde{\sigma}_t(x) - \sigma_t(x)| \le |\tilde{\sigma}_t(x) - \sigma^{(s)}_t(x)| + |\sigma^{(s)}_t(x) - \sigma_t(x)|.
\end{eqnarray}

For the first term,
following the exact same lines as in the proof of Proposition 7 in~\cite{wilson2020efficientsampling}, we have
\begin{eqnarray}
|\tilde{\sigma}^2_t(x)-{\sigma^{(s)}_t}^2(x) | \le C_1m_t |\sigma^2(x) - {\sigma^{(w)}}^2(x) |
\end{eqnarray}
where $C_1 =\max_{1\le t\le T} (1+||K^{-1}_{\Zb_{m_t},\Zb_{m_t}}||_{C(\X^2)})$. 
\cite{wilson2020efficientsampling} proceed to upper bound $|\sigma^2(x) - {\sigma^{(w)}}^2(x) |$ by a constant divided by $\sqrt{M}$. We use a tighter bound based on feature representation of the kernel. Specifically from equation~\eqref{eq:featurerep} and the definition of $\delta_M$ we have that 
\begin{eqnarray}
|\sigma^2(x) - {\sigma^{(w)}}^2(x) | &\le& \sum_{i=M+1}^\infty \lambda_i\psi_i^2\\\nn 
&=& \delta_M,
\end{eqnarray}
which results in the following upper bound
\begin{eqnarray}
|\tilde{\sigma}^2_t(x)-{\sigma^{(s)}_t}^2(x) |  \le C_1m_t \delta_M.
\end{eqnarray}
For the standard deviations we have
\begin{eqnarray}\nn
|\tilde{\sigma}_t(x)-\sigma^{(s)}_t(x) | &=& \sqrt{|\tilde{\sigma}_t(x)-\sigma^{(s)}_t(x) |^2 }\\\nn
&\le&\sqrt{|\tilde{\sigma}_t(x)-\sigma^{(s)}_t(x) ||\tilde{\sigma}_t(x)+\sigma^{(s)}_t(x) | }\\\nn
&=& \sqrt{|\tilde{\sigma}^2_t(x)-{\sigma^{(s)}_t}^2(x) |^2 }\\\label{eq:wilson}
&\le& \sqrt{C_1m_t\delta_M}
\end{eqnarray}
where the first inequality holds because $|\tilde{\sigma}_t(x)-\sigma^{(s)}_t(x) |\le |\tilde{\sigma}_t(x)+\sigma^{(s)}_t(x) |$ for positive $ \tilde{\sigma}_t(x)$ and $\sigma^{(s)}_t(x)$.

We can efficiently bound the error in the SVGP approximation based on the convergence of SVGP methods. Let us first focus on the inducing features. It was shown that (Lemma 2 in~\cite{Burt2019Rates}), for the SVGP with inducing features
\begin{eqnarray}
\text{KL}\left(\text{GP}(\mu_t,\sigma_t), {\text{GP}}(\mu^{(s)}_t,k^{(s)}_t)\right)\le \frac{\theta_t}{\tau}.
\end{eqnarray}
where $\text{GP}(\mu_t,\sigma_t)$ and ${\text{GP}}(\mu^{(s)}_t,k^{(s)}_t)$ are the true and the SVGP approximate posterior distributions at time $t$, and KL denotes the  Kullback-Leibler divergence between them. On the right hand side, $\theta_t$ is the trace of the error in the covariance matrix. Specifically, $\theta_t = \text{Tr}(E_t)$ where $E_t = K_{\Xb_t,\Xb_t} - K_{\Zb_t,\Xb_t}^{\TP}K_{\Zb_t,\Zb_t}K_{\Zb_t,\Xb_t} $. Using the Mercer expansion of
the kernel matrix,~\cite{Burt2019Rates} showed that $[E_t]_{i,i} = \sum_{j=m_t+1}^\infty\lambda_j\phi^2_j(x_i)$. Thus 
\begin{eqnarray}
\theta_t &=&\sum_{i=1}^t\sum_{j=m+1}^\infty\lambda_j\phi^2_j(x_i)\\\nn
&\le& t\sum_{j=m_t+1}^{\infty}\lambda_j\psi_j^2\\\nn
&=&t\delta_{m_t}
\end{eqnarray}

Thus, 
\begin{eqnarray}\label{KL123}
\text{KL}\left(\text{GP}(\mu_t,\sigma_t), {\text{GP}}(\mu^{(s)}_t,k^{(s)}_t)\right)\le \kappa_t/2.
\end{eqnarray}
where $\kappa_t = 2t\delta_m/\tau$. 
In comparison, \cite{Burt2019Rates} proceed by introducing a prior distribution on $x_i$ and bounding $[E_t]_{i,i,}$ differently.

For the case of inducing points drawn from an $\epsilon_0$ close k-DPP distribution, similarly following the exact lines as \cite{Burt2019Rates} except for the upper bound on $[E_t]_{i,i}$, with probability at least $1-\delta$, \eqref{KL123} holds with $\kappa_t = \frac{2t(m_t+1)\delta_{m_t}}{\delta\tau} + \frac{4t\epsilon_0}{\delta\tau}$ where $\epsilon_0$ can be arbitrarily small.

In addition, if the KL divergence between two Gaussian distributions is bounded by $\kappa_t/2$, we have the following bound on their means and variances [Proposition 1 in~\cite{Burt2019Rates}]
\begin{eqnarray}\nn
|\mu^{(s)}_t(x)-\mu_t(x)|&\le& {\sigma}_t(x)\sqrt{\kappa_t},\\\label{eq:burt}
|1-\frac{{\sigma^{(s)}_t}^2(x)}{\sigma^2_t(x)}|&\le&\sqrt{3\kappa_t},
\end{eqnarray}
which by algebraic manipulation gives 
\begin{eqnarray}
\sqrt{1-\sqrt{3\kappa_t}}\sigma_t({x}) \le \sigma^{(s)}_t(x) \le \sqrt{1+\sqrt{3\kappa_t}}\sigma_t({x})
\end{eqnarray}
Combining the bounds on $\sigma^{(s)}_t$ with~\eqref{eq:wilson}, we get
\begin{eqnarray}\nn
\sqrt{1-\sqrt{3\kappa_t}}\sigma_t({x})-\sqrt{C_1m_t\delta_M} \le \tilde{\sigma}_t(x) \le \sqrt{1+\sqrt{3\kappa_t}}\sigma_t({x})+\sqrt{C_1m_t\delta_M}
\end{eqnarray}
Comparing this bounds with Assumption I, we have $\underline{a}_t  =\frac{1}{\sqrt{1-\sqrt{3\kappa_t}}}$, 
$\bar{a}_t = \sqrt{1+\sqrt{3\kappa_t}}$, and $\epsilon_t = \sqrt{C_1m_t\delta_M}$.
Also, since $\mu^{(s)}_t = \tilde{\mu}_t$, comparing~\eqref{eq:burt} with Assumption II, we have $c_t = \sqrt{\kappa_t}$. 

$\square$



\subsection*{Proof of Theorem~\ref{The:MatSE}}

Theorem~\ref{The:RegretBoundsApproximate} proved that $R(T;\text{S-GP-TS}) = {O}\left( \underline{a}\bar{a}\gamma_T\sqrt{ T\log(T)} + \underline{a}\epsilon T\sqrt{\gamma_T\log(T)}  \right)$. We thus need to show that $\underline{a}\bar{a}$ is a constant independent of $T$ and $\underline{a}\epsilon$ is small so that the second term is dominated by the first term. 

In the case of Mat{\'e}rn kernel, $\lambda_j = O(j^{-\frac{2\nu+d}{d}})$ implies that $\delta_m = O(m^{-\frac{2\nu}{d}})$. Under sampling rule~\eqref{SampIndPoints}, we select $\delta=\frac{1}{T}$ and $\epsilon_0=\frac{1}{T^2\log(T)}$ in Proposition~\ref{Prop:paramters}. We thus need $\kappa_T = O({T^2m_T\delta_{m_T}})$ and $\epsilon_T\sqrt{T} = O(\sqrt{m_T\delta_M T}) $ be sufficiently small constants. That is achieved by selecting $m_T = T^{\frac{2d}{2\nu-d}}$ and $M=T^{\frac{(2\nu+d)d}{2(2\nu-d)\nu}}$.

Under sampling rule~\eqref{SampIndVariables}, we need $\kappa_T = O({T\delta_{m_T}})$ and $\epsilon_T\sqrt{T} = O(\sqrt{m_T\delta_M T}) $ be sufficiently small constants. That is achieved by selecting $m_T = T^{\frac{d}{2\nu}}$ and $M=\frac{(2\nu+d)d}{4\nu^2 }$.

In the case of SE kernel, $\lambda_j = O(\exp(-j^{\frac{1}{d}}))$ implies that $\delta_m = O(\exp(-m^{\frac{1}{d}}))$. Under sampling rule~\eqref{SampIndPoints}, we select $\delta=\frac{1}{T}$ and $\epsilon_0=\frac{1}{T^2\log(T)}$ in Proposition~\ref{Prop:paramters}. We thus need $\kappa_T = O({T^2m_T\delta_{m_T}})$ and $\epsilon_T\sqrt{T} = O(\sqrt{m_T\delta_M T}) $ be sufficiently small constants. That is achieved by selecting $m_T = (\log(T))^{d}$ and $M=(\log(T))^{d}$.
We obtain the same results under 
sampling rule~\eqref{SampIndVariables} where we need $\kappa_T = O({T\delta_{m_T}})$ and $\epsilon_T\sqrt{T} = O(\sqrt{m_T\delta_M T}) $ be sufficiently small constants. 

$\square$
